\pdfoutput=1
\documentclass[11pt,a4paper]{amsart}
\usepackage{amsaddr}
\usepackage{amsfonts}
\usepackage{amsthm}
\usepackage{amscd}
\usepackage[utf8]{inputenc}
\usepackage[mathscr]{eucal}
\usepackage{indentfirst}
\usepackage{graphicx}
\numberwithin{equation}{section}
\usepackage[margin=2.9cm]{geometry}

\usepackage{dsfont}
\usepackage{macros}

\newcommand{\LSEM}{\operatorname{LSEM}}
\newcommand{\Rel}[2]{\operatorname{Rel}(#1, #2)}

\newcommand{\GAM}{n^6}

\newcommand{\DEL}{5}

\usepackage{fullpage}
\newcommand{\omegamatrix}[2]{\Omega_{#1,#2}}
\newcommand{\mui}[2]{\Lambda_{#1,#2}}
\newcommand{\muone}{\Lambda_{1,2}}
\newtheorem{model}[theorem]{Model}

\newcommand{\Diag}{\operatorname{\alpha}}

\begin{document}

\title{Stability of Linear Structural Equation Models of Causal Inference}

\author{ Karthik Abinav Sankararaman \and
Anand Louis \and
Navin Goyal
}
\address{
University of Maryland, College Park\and
Indian Institute of Science, Bangalore \and
Microsoft Research, India
} 
\email{
kabinav@cs.umd.edu \and
anandl@iisc.ac.in \and
navingo@microsoft.com
}

\maketitle


\begin{abstract} 
We consider the numerical stability of the parameter recovery problem in Linear Structural Equation Model ($\LSEM$) of causal inference. A long line of work starting from Wright (1920) has focused on understanding which sub-classes of $\LSEM$ allow for efficient parameter recovery. Despite decades of study, this question is not yet fully resolved. The goal of this paper is complementary to this line of work; we want to understand the stability of the recovery problem in the cases when efficient recovery is possible. Numerical stability of Pearl's notion of causality was first studied in Schulman and Srivastava (2016) using the concept of condition number where they provide ill-conditioned examples. In this work, we provide a condition number analysis for the $\LSEM$. First we prove that under a sufficient condition, for a certain sub-class of $\LSEM$ that are \emph{bow-free} (Brito and Pearl (2002)), the parameter recovery is stable. We further prove that \emph{randomly} chosen input parameters for this family satisfy the condition with a substantial probability. Hence for this family, on a large subset of parameter space, recovery is numerically stable. Next we construct an example of $\LSEM$ on four vertices with \emph{unbounded} condition number. We then corroborate our theoretical findings via simulations as well as real-world experiments for a sociology application. Finally, we provide a general heuristic for estimating the condition number of any $\LSEM$ instance. 
\end{abstract}

\maketitle
\section{Introduction}

Inferring \emph{causality}, \ie whether a group of events causes another group of events is a central problem in a wide range of fields from natural to social sciences. A common approach to inferring causality is Randomized controlled trials (RCT). Here the experimenter intervenes on a system of variables (often called stimulus variables) such that it is not affected by any confounders with the variables of interest (often called response variables) and observes the probability distributions on the response variables. Unfortunately, in many cases of interest performing RCT is either costly or impossible due to practical or ethical or legal reasons. A common example is the age-old debate \cite{108} on whether smoking causes cancer. In such scenarios RCT is completely out of the question due to ethical issues. This necessitates new inference techniques.

	The causal inference problem has been extensively studied in statistics and mathematics (\emph{e.g.,} \cite{64,67,pearlBook,95}) where decades of research has led to rich mathematical theories and a framework for conceptualizing and analyzing causal inference. One such line of work is the \emph{Linear Structural Equation Model} (or $\LSEM$ in short) for formalizing causal inference (see the monograph \cite{Bollen} for a survey of classical results). In fact, this is among the most commonly used models of causality in social sciences~\cite{25,Bollen} and some natural sciences~\cite{spirtes2010introduction}. In this model, we are given a mixed graph on $n$ (observable) variables\footnote{In this paper we are interested in properties for large $n$.} of the system containing both directed and bi-directed edges (see Figure~\ref{fig:mixedGraph} for an example). 
	We will assume that the directed edges in the mixed graph form a directed acyclic graph (DAG). A directed edge from vertex $u$ to vertex $v$ represents the presence of causal effect of variable $u$ on variable $v$, while the bi-directed edges represent the presence of confounding effect (modeled as noise) which we next explain.\footnote{We also interchangeably refer to the directed edges as \emph{observable} edges since they denote the direct causal effects and the bi-directed edges as \emph{unobservable} edges since they indicate the unobserved common causes.} In the $\LSEM$, the following extra assumption is made (see Equation~\eqref{eq:LinearSEM}): the value of a variable $v$ is determined by a (weighted) linear combination of its parents' (in the directed graph) values added with a zero-mean \emph{Gaussian} noise term ($\eta_v$). The bi-directed graph indicates dependencies 
	between the noise variables (\ie lack of an edge between variables $u$ and $v$ implies that the covariance between $\eta_u$ and $\eta_v$ is $0$). We use $\mathbf{\Lambda} \in \mathbb{R}^{n \times n}$ to represent the matrix of edge weights of the DAG, $\mathbf{\Omega} \in \mathbb{R}^{n \times n}$ to represent the covariance matrix of the Gaussian noise and $\mathbf{\Sigma} \in \mathbb{R}^{n \times n}$ to represent the covariance matrix of the observation data (henceforth called \emph{data covariance matrix}). Let $\mathbf{X} \in \mathbb{R}^{n \times 1}$ denote a vector of random variables corresponding to the observable variables in the system. Let $\mathbf{\eta} \in \mathbb{R}^{n \times 1}$ denote the vector of corresponding noises whose covariance matrix is given by $\mathbf{\Omega}$. Formally, the $\LSEM$ assumes the following relationship between the random variables in $\mathbf{X}$:
		\begin{equation}
			\label{eq:LinearSEM}
			\mathbf{X} = \mathbf{\Lambda}^T \mathbf{X} + \mathbf{\eta}.
		\end{equation}
	From the \emph{Gaussian} assumption on the noise random variable $\mathbf{\eta}$, it follows that $\mathbf{X}$ is also a multi-variate Gaussian with covariance matrix given by,
		\begin{equation}
			\label{eq:sigmaEq}
			\mathbf{\Sigma} = \paren{\mathbf{I} - \mathbf{\Lambda}}^{-T} \mathbf{\Omega} \paren{\mathbf{I} - \mathbf{\Lambda}}^{-1} \mper 
		\end{equation}
	\begin{figure}
	\centering
	\includegraphics[scale=0.3]{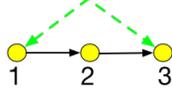}
	\caption{Mixed Graph: Black solid edges represent causal edges. Green dotted edges represent covariance of the noise.}
	\label{fig:mixedGraph}
	\end{figure}
	In a typical setting, the experimenter estimates the joint-distribution by estimating the covariance matrix $\mathbf{\Sigma}$ over the observable variables obtained from finitely many samples. The experimenter also has a causal hypothesis (represented as a mixed graph for the causal effects and the covariance among the noise variables, which in turn determines which entries of $\mathbf{\Lambda}$ and $\mathbf{\Omega}$ are required to be $0$, referred to as the \emph{zero-patterns} of $\vec{\Lambda}$ and $\vec{\Omega}$). One then wants to solve the inverse problem of recovering $\mathbf{\Lambda}$ and $\mathbf{\Omega}$ given $\mathbf{\Sigma}$. This problem is solvable for some special types of mixed graphs using parameter recovery algorithms, such as the one in \cite{RICF,FDD2012Annals}.
	
	Thus a central question in the study of $\LSEM$ is for which mixed graphs (specified by their zero patterns) and which values of the parameters ($\mathbf{\Lambda}$ and $\mathbf{\Omega}$) is the inverse problem above solvable; in other words, which parameters are identifiable. Ideally one would like all values of the parameters to be identifiable. However, identifiability is often too strong a property to expect to hold for all parameters and instead we are satisfied with a slightly weaker property, namely \emph{generic identifiability} (GI): here we require that identifiability holds for all parameter values except for a measure $0$ set (according to some reasonable measure). 
	(The issue of identifiability is a very general one that arises in solving inverse problems in statistics and many other areas of science.)
	A series of works (\emph{e.g.,} \cite{BP2006UAI,DW2016Scandinavian,51,FDD2012Annals,85}) have made progress on this question by providing various classes of mixed graphs that do allow generic identifiability. However, the general problem of generic identifiability has not been fully resolved. This problem is important since 
	it is a version of a central question in science: what kind of causal hypotheses can be validated purely from observational data as opposed to the situations where one can do RCT. Much of the prior work has focused on designing algorithms with the assumption that the \emph{exact} joint distribution over the variables is available. However, in practice, the data is noisy and inaccurate and the joint distribution is generated via \emph{finite} samples from this noisy data. 
	
	While theoretical advances assuming exact joint distribution have been useful it is imperative to understand the effect of violation of these assumptions rigorously. Algorithms for identification in $\LSEM$s are \emph{numerical} algorithms that solve the inverse problem of recovering the underlying parameters constructed from noisy and limited data. Such models and associated algorithms are useful only if they solve the inverse problem in a \emph{stable} fashion: if the data is perturbed by a small amount then the recovered parameters change only a small amount. If the recovery is unstable then the recovered parameters are unlikely to tell us much about the underlying causal model as they are inordinately affected by the noise. We say that \emph{robust identifiability} (RI) holds for parameter values $\mathbf{\Lambda}$ and $\mathbf{\Omega}$ if even after 
	perturbing the corresponding $\mathbf{\Sigma}$ to $\mathbf{\Sigma}'$ (by a small noise), the recovered parameters values 
	$\mathbf{\Lambda}'$ and $\mathbf{\Omega}'$ are close to the original ones. 
	It follows from the preceding discussion that to consider an inverse problem solved it is not sufficient for generic identifiability to hold; instead, we would like the stronger property of robust identifiability to hold for almost all parameter values (we call this generic robust identifiability; for now we leave the notions of ``almost everywhere'' and ``close'' informal). In addition, the problem should also be solvable by an efficient algorithm. The mixed graphs we consider all admit efficient parameter recovery algorithms. Note that GI and RI are properties of the problem and not of the algorithm used to solve the problem. 
	 
In the general context of inverse problems, the difference between GI and RI is quite common and important: \emph{e.g.,} in solving a system of $n$ linear equations in $n$ variables given by an $n \times n$ matrix $M$. For any reasonable distribution on $n \times n$ matrices, the set of singular matrices has measure $0$ (an algebraic set of lower dimension given by $\det(M)=0$). Hence, $M$ is invertible with probability $1$ and GI holds. This however does not imply that generic robust identifiability holds: for that one needs to say that the set of ill-conditioned matrices has small measure. To understand RI for this problem, one needs to resort to analyses of the minimum singular value of $M$ which are non-trivial in comparison to GI (e.g., see \cite[Sec 2.4]{burgisser2013condition}). In general, proving RI almost everywhere turns out to be harder and remains an open problem in many cases even though GI results are known. One recent example is rank-1 decomposition of tensors (for $n \times n \times n$ random tensors of rank up to $n^2/16$, GI is known, whereas generic robust identifiability is known only up to rank $n^{1.5}$); see, \emph{e.g.,} \cite{BhaskaraCMV14}. The problem of tensor
decomposition is just one example of a general recent concern for RI results in the theoretical computer science literature and there are many more examples. Generic robust identifiability remains an open problem for semi-Markovian models for which efficient GI is known, \emph{e.g.,} \cite{pearlBook}.

In the context of causality, the study of robust identifiability was initiated in \cite{SS2016UAI} where the authors construct a family of examples in the so-called Pearl's notion of causality on 
	semi-Markovian graphs\footnote{Unlike $\LSEM$, this is a non-parametric model. The functional dependence of a variable on the 
		parents' variables is allowed to be fully general, in particular it need not be linear. This of course comes at the price of making
		the inference problem computationally and statistically harder.}
		 (see \emph{e.g.,} \cite{pearlBook}) and show that for this family there exists an \emph{adversarial} perturbation of the input which causes the recovered parameters to be drastically different (under an appropriate metric described later) from the actual set of parameters. However this result has the following limitation. Their adversarial perturbations are carefully crafted and this worst case scenario can be alleviated by modifying just a few edges (in fact just deleting some edges randomly suffices which can also be achieved without changing the zero-pattern by choosing the parameters appropriately). 
		 This leads us to ask: \emph{how prevalent are such perturbations?} Since
	there is no canonical notion of what a typical LSEM model is (i.e. the graph structure and the associated parameters), we will assume that the graph structure is given and the parameters are randomly chosen according to some reasonable distribution. Thus we would like to answer the following question\footnote{The question of understanding the condition number (a measure of stability) of LSEMs was raised in \cite{SS2016UAI}, though presumably in the worst-case sense of whether there are identifiable instances for which recovery is unstable. As mentioned in their work, when the model is \emph{not} uniquely identifiable, the authors in \cite{cornia2014type} show an example where uncertainty in estimating the parameters can be unbounded.}.
	
	\begin{question}[Informal]
		\label{que:informal}
		For the class of $\LSEM$s that are uniquely identifiable, does robust identifiability hold for most choices of parameters?
	\end{question}
The question above is informal mainly because ``most choices of parameters'' is not a formally defined notion. We can quantify this notion by putting a probability measure on the space of all parameters. This is what we will do.

\xhdr{Notation.} Throughout this paper, we will use the following notation. Bold fonts represent matrices and vectors. We use the following shorthand in proofs.
			\begin{equation}
				\label{eq:Lambdaij}
				\Lambda_{\ell,k} = \begin{cases} \Pi_{j = \ell}^{k-1} \Lambda_{j,j+1} & \textrm{if } \ell < k \\ 
						1 & \textrm{if }\ell \geq k \end{cases}.
			\end{equation}
Given matrices $\mathbf{A}, \mathbf{B} \in \mathbb{R}^{n \times m}$, we define the \emph{relative distance}, denoted by $\Rel{\mathbf{A}}{\mathbf{B}}$ as the following quantity.
		\[
			\Rel{\mathbf{A}}{\mathbf{B}} \defeq \max_{1 \leq i \leq n, 1 \leq j \leq m : \Abs{A_{i, j}} \neq 0} \frac{|A_{i, j}-B_{i, j}|}{\Abs{A_{i, j}}}.
		\]
		In this paper we use the notion of condition number (see \cite{burgisser2013condition} for a detailed survey on condition numbers in numerical algorithms) to quantitatively measure the effect of perturbations on data in the parameter recovery problem. The 
		specific definition of condition number we use is a natural extension of the $\ell_{\infty}$-condition number studied in \cite{SS2016UAI} to matrices.
		
		\begin{definition}[Relative $\ell_{\infty}$-condition number]
			\label{def:linftyCondition}
			Let $\mathbf{\Sigma}$ be a given data covariance matrix and $\mathbf{\Lambda}$ be the corresponding parameter matrix. Let a $\gamma$-perturbed family of matrices be denoted by $\mathcal{F}_{\gamma}$  (\emph{i.e.,} set of matrices $\mathbf{\tilde{\Sigma}}_{\gamma}$ such that $\Rel{\mathbf{\Sigma}}{\tilde{\mathbf{\Sigma}}_\gamma} \leq \gamma $). For any $\mathbf{\tilde{\Sigma}}_{\gamma} \in \mathcal{F}_{\gamma}$ let the corresponding recovered parameter matrix be denoted by $\mathbf{\tilde{\Lambda}}_{\gamma}$. Then the relative $\ell_{\infty}$-condition number is defined as, 
			\begin{equation}
				\label{eq:linftyCondition}
				\kappa(\mathbf{\Lambda}, \mathbf{\Sigma}) \defeq \lim_{\gamma \rightarrow 0^+} \sup_{\tilde{\mathbf{\Sigma}}_\gamma \in \mathcal{F}_{\gamma}} \frac{\Rel{\mathbf{\Lambda}}{\tilde{\mathbf{\Lambda}}_\gamma}}{\Rel{\mathbf{\Sigma}}{\tilde{\mathbf{\Sigma}}_\gamma} }.
			\end{equation}
 		\end{definition}

We confine our attention to the stability of recovery of $\mathbf{\Lambda}$ as once $\mathbf{\Lambda}$ is recovered approximately, 
$\mathbf{\Omega} = \paren{\mathbf{I} - \mathbf{\Lambda}}^{T}  \mathbf{\Sigma} \paren{\mathbf{I} - \mathbf{\Lambda}}$ can be easily approximated in a stable manner.
In this paper, we restrict our theoretical analyses to causal models specified by \emph{bow-free} paths which are inspired by bow-free graphs \cite{brito2002new}\footnote{See footnote 4 in \cite{brito2002new} for the significance of bow-free graphs in $\LSEM$.}. In \cite{brito2002new} the authors show that the bow-free property\footnote{Bow-free (mixed) graphs are those where for any pair of vertices, both directed and undirected edges are never present simultaneously. In other words, any pair of variables with direct causal relationship (implied by a directed edge connecting the two variables) do not have correlated errors (not connected by a bi-directed edge)}  of the mixed graph underlying the causal model is a sufficient condition for unique identifiability. We now define bow-free paths; Figure~\ref{fig:illustration} has an example.

\begin{definition}[Bow-free Paths]
	\label{def:paths}
	A causal model is called a \emph{bow-free path} if the underlying DAG forms a directed path and the mixed graph is bow-free~\cite{brito2002new}. Consider $n$ vertices indexed $1, 2, \ldots, n$. The (observable) DAG forms a path starting at $1$ with directed edges going from vertex $i$ to vertex $i+1$. The bi-directed edges can exist between pairs of vertices $(i, j)$ only if $|i-j| \geq 2$. Thus the only potential non-zero entries in $\mathbf{\Lambda}$ are in the diagonal immediately above the principal diagonal. Similarly, the diagonal immediately above and below the principal diagonal in $\mathbf{\Omega}$ is always zero.\footnote{We want to emphasize that bow-free paths allow \emph{all} pairs, which do not have a causal directed edge, to have bi-directed edges.} 
\end{definition}
\begin{figure}
	\centering
	\includegraphics[scale=0.3]{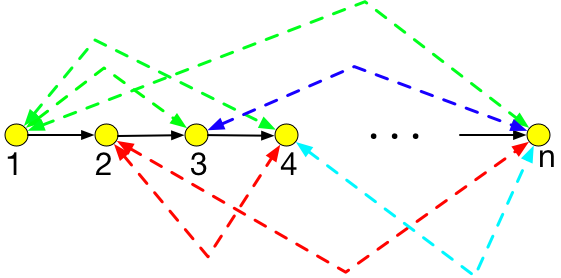}
	\caption{Illustration of a bow-free path. Solid lines represent directed causal edges and dotted lines represent bi-directed covariance edges.}
	\label{fig:illustration}
\end{figure}
We further assume the following conditions on the parameters of the \emph{bow-free} paths for studying the numerical stability of \emph{any} parameter recovery algorithm. Later in Section~\ref{sec:random} we show that a natural random process of generating $\mathbf{\Lambda}$ and $\mathbf{\Omega}$ satisfies these assumptions with high-probability. 
Informally, the model can be viewed as a \emph{generalization} of \emph{symmetric diagonally dominant} (SDD) matrices.

\begin{model}
	\label{mod:localDominance}
	Our model satisfies the following conditions on the data-covariance matrix and the perturbation with the underlying graph structure given by bow-free paths. It takes parameters $\alpha$ and $\lambda$. 
	
	\begin{enumerate}
		\item 
		\xhdr{Data properties.}
		$\mathbf{\Sigma} \succeq 0$ is an $n \times n$ symmetric matrix satisfying the following conditions\footnote{We note that the important part is that $\Diag$ is bounded by a \emph{constant} independent of $n$. The precise constant is of secondary importance.}
	for some $0 \leq \Diag \leq \frac{1}{\DEL}$.
	\[ \Abs{\Sigma_{i-1, i}}, \Abs{\Sigma_{i, i+1}}, \Abs{\Sigma_{i-1, i+1}}
	\leq \Diag \Sigma_{i,i} \qquad \forall i \in [n-1]. \]
	Additionally, the \emph{true} parameters $\mathbf{\Lambda}$, corresponding to $\mathbf{\Sigma}$, satisfy the following. For every $i, j$ such that there is a directed edge from $i$ to $j$, we have $\frac{1}{n^{2}} \leq \frac{1}{\lambda} \leq \Abs{\Lambda_{i, j}} \leq 1$, where $\lambda$ is a parameter. 
	\item 
	\xhdr{Perturbation properties.}
	For each $i,j \in [n]$ and a fixed $\gamma \leq \frac{1}{\GAM}$, let $\e_{i,j} = \e_{j,i}$ be arbitrary numbers satisfying
	$0 \leq \Abs{\e_{i,j}} = \Abs{\e_{j,i}} \leq \gamma \Abs{\Sigma_{i,j}}$ for all $i,j$ such that that there exists a pair $i^*, j^* \in [n]$ with $\Abs{\e_{i^*,j^*}} = \Abs{\e_{j^*,i^*}} = \gamma \Abs{\Sigma_{i^*, j^*}}$. Note that we eventually consider $\gamma$ in the limit going to $0$, hence some small but fixed $\gamma$ suffices. Let $\tilde{\Sigma}_{i,j} \defeq \Sigma_{i,j} + \e_{i,j}$ for every pair $(i, j)$. 
	\end{enumerate}
\end{model}

\begin{remark}
	 The covariance matrix has errors arising from three sources: (1) measurement errors, (2) numerical errors, and (3) error due to finitely many samples. The combined effect is modeled by Item 2 (perturbation properties) above of Model 1.4 which is very general as we only care about the max error. 
\end{remark}

\begin{remark}
	As shown later (see Equation~\eqref{eq:claimRemark}) in the proof of Theorem~\ref{thm:main2}, the constant $\frac{1}{\DEL}$ is an \emph{approximation} to the following. Let $\tau = 1 + 5n\gamma$. Then we want $[1-(\tau+2)\Diag - (\tau + 1) \Diag^2] > 0$ and $\Diag \leq \frac{1}{4}$.
\end{remark}
	
With the definitions and model in place, we can now pose Question~\ref{que:informal} formally:

\begin{question}[Formal]
		\label{que:formal}
		For the class of $\LSEM$s represented by bow-free paths and Model~\ref{mod:localDominance}, can we characterize the behavior of the $\ell_{\infty}$-condition number?
	\end{question}

\subsection{Our Contributions}

Our key contribution in this paper is a step towards understanding Question~\ref{que:informal} by providing an answer to Question~\ref{que:formal}. In particular, we first prove that when Model assumptions~\ref{mod:localDominance} hold, the $\ell_{\infty}$-condition number of the parameter recovery problem is upper-bounded by a polynomial in $n$. Formally, we prove Theorem~\ref{thm:main2}. This implies that the loss in precision scales \emph{logarithmically} and hence we need at most $O(\log d)$ additional bits to get a precision of $d$ bits\footnote{Note we already need $O(\poly(n) \log n)$ bits to represent the graph and the associated matrices.}. See \cite{burgisser2013condition} and \cite{SS2016UAI} for further discussion on the relation between condition number and the number of bits needed for $d$-bit precision in computation as well as other implications of small condition number.

	\begin{theorem}[Stability Result]
	\label{thm:main2}
		Under the assumptions in Model~\ref{mod:localDominance}, we have the following bound on the condition number for bow-free paths with $n$ vertices.
	\[
			\kappa(\vec{\Lambda}, \vec{\Sigma}) \leq O(n^{2}).
	\] 
	\end{theorem}
	
	More specifically, the condition number is upper-bounded by $\kappa(\vec{\Lambda}, \vec{\Sigma}) \leq O \left( \lambda \right)$ and for the parameters chosen in this model, we have the bound in Theorem~\ref{thm:main2}.
	
	Furthermore, in Section~\ref{sec:random} we show that a natural generative process to construct matrices $\mathbf{\Lambda}$ and $\mathbf{\Omega}$ satisfies the Model assumptions~\ref{mod:localDominance} \footnote{We in fact study various regimes in the parameter space of the generative process and show in which of those regimes the model assumptions hold.}. Hence this implies that a \emph{large} class of instances are well-conditioned and hence ill-conditioned models are not prevalent under our generative assumption. Moreover, as described in the model preliminaries, all data covariance matrices that are SDD and have the property that every row has at least $8$ entries that are not arbitrarily close to $0$, satisfy the assumptions in Model~\ref{mod:localDominance}. Thus an important corollary of this theorem is that when the data covariance matrix is in this class of SDD matrices, it is well-conditioned.

	Next we show that there exist examples for $\LSEM$s with \emph{arbitrarily} high condition number. This implies that on such examples it is unreasonable to expect any form of accurate computation. 
	Formally, we prove Theorem~\ref{thm:mainIns}. It shows that the recovery problem itself has a bad condition number and does not depend on any particular algorithm that is used for parameter recovery. This theorem follows easily using the techniques of \cite{FDD2012Annals} and we include it for completeness.
	
	\begin{theorem}[Instability Result]
		\label{thm:mainIns}
		There exists a bow-free path of length four and data covariance matrix $\mathbf{\Sigma}$ such that the parameter recovery problem of obtaining $\mathbf{\Lambda}$ from $\mathbf{\Sigma}$ has an arbitrarily large condition number.
	\end{theorem}

	We perform numerical simulations to corroborate the theoretical results in this paper. We verify our theorems using numerical simulations. Furthermore, we consider general graphs (\emph{e.g.,} clique of paths, layered graphs) and show that a similar behavior on the condition number holds. Finally, we consider a real-world dataset used in \cite{shimizu2011directlingam} and perform condition number analysis experimentally. 
	
	We conclude the paper by giving a general heuristic for practitioners to determine the condition number of any given instance. This heuristic can detect \emph{bad} instances with high-probability. This procedure might be of independent interest. 

\subsection{Related work}
A standard reference on Structural Causal Models is Bollen~\cite{Bollen}. Identifiability and robust identifiability of $\LSEM$s has been studied from various viewpoints in the literature: Robustness to model misspecification, to measurement errors, \emph{etc.} (\emph{e.g.,} Chapter 5 of Bollen~\cite{Bollen} and \cite{Hancock, Satorra, Huang_etal}). These works are not directly related to this paper, since they focus on the identifiability problem under erroneous model specification and/or measurement. See the sub-section on measurement errors below for further details.

\xhdr{Identifiability.} The problem of (generic) identification in $\LSEM$s is well-studied though still not fully understood. We give a brief overview of the known results. All works in this section assume that the data-covariance matrix is exact and do not consider robustness aspects of the problem. These works do not directly relate to the problem studied in this paper. This problem has a rich history (see~\cite{pearlBook} for some classical results) and we only give an overview of recent results. Brito and Pearl \cite{BP2006UAI} gave a sufficient condition called G-criterion which implies linear independence on a certain set of equations. The variables involved in this set is called the auxiliary variables. The main theorem in their paper is that if for every variable there is a corresponding ``auxiliary'' variable, then the model is identifiable. Followed by this, Foygel \etal~\cite{FDD2012Annals} introduced the notion of half-trek criterion which gives a graphical criterion for identifiability in $\LSEM$. This criterion strictly subsumes the G-criterion framework. Both \cite{BP2006UAI,FDD2012Annals} supply 
efficient algorithms for identification. Chen \etal~\cite{CTP2014AAAI} tackle the problem of identifying ``overidentifying constraints'', which leads to identifiability on a class of graphs strictly larger than those in \cite{FDD2012Annals} and \cite{BP2006UAI}. Ernest \etal~\cite{ERB2016Arxiv} consider a variant of the problem called ``Partially Linear Additive SEM'' (PLSEM) with gaussian noise. In this variant the value for a random variable $X$ is determined both by an additive linear factor of some of its parents (called linear parents) and additive factor of the other parents (called non-linear parents) via a non-linear function. They give characterization for identifiability in terms of graphical conditions for this more general model. Chen~\cite{Chen2016aNIPS} extends the half-trek criterion of \cite{FDD2012Annals} and give a $c$-component decomposition (Tian~\cite{Tian2012}) based algorithm to identify a broader class of models. Drton and Weihs~\cite{DW2016Scandinavian} also extend the half-trek criterion (\cite{FDD2012Annals}) to include the ancestor decomposition technique of Tian~\cite{Tian2012}. Chen \etal~\cite{CKB2017ICML} approach the parameter identification problem for $\LSEM$ via the notion of \emph{auxiliary variables}. This method subsumes all the previous works above and identifies a strictly larger set of models.

	\xhdr{Measurement errors.} Another recent line of work \cite{scheines2016measurement, zhang2017causal} studies the problem of identification under \emph{measurement errors}. Both our work and these works share the same motivation: causal inference in real-world is usually performed on noisy data and hence it is important to understand how noise affects the causal identification. However their focus differs significantly from the question we tackle in this paper. They pose and answer the problem of causal identification when the variables used are not identical to the ones that were intended to be measured. This leads to different conditional independences and hence a different causal graph; they study the identification problem in this setting and characterize when such identification is possible. Recently \cite{ghoshal2018learning} looked at sample complexity of identification in $\LSEM$. They consider the special case when $\vec{\Omega}$ is the Identity matrix and give a new algorithm for identifiability with optimal sample complexity.

	In this paper we are interested in the question of robust identifiability, along the lines of aforementioned work of Schulman and Srivastava~\cite{SS2016UAI}. While the work of \cite{SS2016UAI} was direct inspiration for the present paper, since we work with LSEMs and \cite{SS2016UAI} work with the semi-Markovian causal models of Pearl, the techniques involved are completely different. Moreover, as previously remarked, another important difference between \cite{SS2016UAI} and our work is that our main result is positive: the
	parameter recovery problem is well-conditioned for most choices of the parameters for a well-defined notion of most choices.

\section{Bow-free Paths --- Stable Instances}

In this section, we show that for bow-free paths under Model~\ref{mod:localDominance}, the condition number is small. Later in Section~\ref{sec:random} we show that a natural random process of generating $\mathbf{\Lambda}$ and $\mathbf{\Omega}$ satisfies these assumptions. Together this implies that for a large class of inputs, the stability of the map is well-behaved. More precisely, we prove Theorem~\ref{thm:main2}.

To prove the main theorem, we set-up some helper lemmas. Using Foygel \etal~\cite{FDD2012Annals}, we obtain the following recurrence to compute the values of $\mathbf{\Lambda}$ from the correlation matrix $\mathbf{\Sigma}$. The base case is $\Lambda_{1, 2} = \frac{\Sigma_{1, 2}}{\Sigma_{1, 1}}$. For every $i \geq 2$ we have,
\begin{equation}  
\label{eqn:Recurrence}
\mui{i}{i+1} = \frac{-\mui{i-1}{i} \Sigma_{i-1,i+1} + \Sigma_{i,i+1}}
{-\mui{i-1}{i} \Sigma_{i-1,i} + \Sigma_{i,i}} \mper
\end{equation}
We first show that the model assumptions~\ref{mod:localDominance} imply that the parameters recovered from the perturbed version are not too large.

\begin{lemma}
	\label{lem:muipbounded}
	For each $i \in [n-1]$, we have 
	$\Abs{\tilde{\Lambda}_{i, i+1}} \leq \tau_i ~\left(:= \tau_{i-1}  + 5 \gamma \right) \leq \tau \left(:= 1  + 5n \gamma \right)$.
\end{lemma}

\begin{proof}
	For $i=1$, we have
	\[ \Abs{\tilde{\Lambda}_{1, 2}} = \Abs{\frac{\tilde{\Sigma}_{1,2}}{\tilde{\Sigma}_{1,1}}} 
	\leq \frac{\Abs{\Sigma_{1, 2}} + \gamma \Sigma_{1,1}}{\Sigma_{1, 1} - \gamma \Sigma_{1, 1}} \leq \frac{\Diag + \gamma}{1 - \gamma} \leq \frac{1}{\DEL}+ 2 \gamma\]
	The last inequality used the fact that $\frac{1}{1-\gamma} \leq 2$ when $n \geq 2$. Indeed, this is true since $\gamma \leq \frac{1}{\GAM}$. Thus, $\gamma \leq \frac{1}{64}$ when $n \geq 2$. Therefore, $\frac{1}{1-\gamma} \leq \frac{64}{63} \leq 2$. Thus, we have
	
	\[ \frac{1}{\DEL} + 2 \gamma \leq 1 + 2 \gamma (:= \tau_{1}) \leq \tau \quad \paren{\textrm{Using } \Diag \leq \frac{1}{\DEL}} \mper \]
	Suppose, $\Abs{\tilde{\Lambda}_{i-1, i}} \leq \tau_{i-1}$. 
	\begin{align}
		 \Abs{\tilde{\Lambda}_{i, i+1}} & = \frac{\Abs{-\tilde{\Lambda}_{i-1, i} \tilde{\Sigma}_{i-1,i+1} + \tilde{\Sigma}_{i,i+1}}}
	{\Abs{-\tilde{\Lambda}_{i-1, i} \tilde{\Sigma}_{i-1,i} + \tilde{\Sigma}_{i,i}}}	& \text{(From recurrence~\eqref{eqn:Recurrence})} \label{eq:firstEq} \\
	& \leq \frac{\Abs{\tilde{\Sigma}_{i,i+1}} + \Abs{\tilde{\Lambda}_{i-1, i}}\Abs{\tilde{\Sigma}_{i-1,i+1}}}{\Abs{-\tilde{\Lambda}_{i-1, i} \tilde{\Sigma}_{i-1,i} + \tilde{\Sigma}_{i,i}}} &\nonumber \\
	& \leq \frac{\Abs{\tilde{\Sigma}_{i,i+1}} + \tau_{i-1} \Abs{\tilde{\Sigma}_{i-1,i+1}}}{\Abs{-\tilde{\Lambda}_{i-1, i} \tilde{\Sigma}_{i-1,i} + \tilde{\Sigma}_{i,i}}} & \text{(From Inductive Hypothesis)} \label{eq:IH} 
	\end{align}
	We will now prove a lower-bound on the denominator $\Abs{-\tilde{\Lambda}_{i-1, i} \tilde{\Sigma}_{i-1,i} + \tilde{\Sigma}_{i,i}}$. For every $x, y$ we know that $\Abs{x + y} \geq \Abs{\Abs{x} - \Abs{y}}$. Thus, it follows that $\Abs{-\tilde{\Lambda}_{i-1, i} \tilde{\Sigma}_{i-1,i} + \tilde{\Sigma}_{i,i}} \geq \Abs{\Abs{\tilde{\Sigma}_{i,i}} - \Abs{\tilde{\Lambda}_{i-1, i} \tilde{\Sigma}_{i-1,i}}}$. 
	
	We will now show that $\Abs{\tilde{\Sigma}_{i,i}} - \Abs{\tilde{\Lambda}_{i-1, i} \tilde{\Sigma}_{i-1,i}} \geq 0$. Thus, we have $\Abs{\Abs{\tilde{\Sigma}_{i,i}} - \Abs{\tilde{\Lambda}_{i-1, i} \tilde{\Sigma}_{i-1,i}}} = \Abs{\tilde{\Sigma}_{i,i}} - \Abs{\tilde{\Lambda}_{i-1, i} \tilde{\Sigma}_{i-1,i}}$.
	
Using the data properties of the model and the inductive hypothesis, we have $\Abs{\tilde{\Sigma}_{i, i}} \geq \Sigma_{i, i}(1 - \gamma)$, $\Abs{\tilde{\Lambda}_{i-1, i}} \leq \tau_{i-1}$ and $\Abs{\tilde{\Sigma}_{i-1, i}} \leq \Abs{\Sigma_{i-1, i}} + \gamma \Sigma_{i, i} \leq \Sigma_{i, i}(\Diag + \gamma)$. Thus, we have 
	\[
			\Abs{\tilde{\Sigma}_{i,i}} - \Abs{\tilde{\Lambda}_{i-1, i} \tilde{\Sigma}_{i-1,i}} \geq \Sigma_{i, i} \Paren{1-\tau_{i-1} \Diag - (1+\tau_{i-1}) \gamma} \geq \Sigma_{i, i} \Paren{1- (1 + 5n \gamma)\Diag - 3\gamma} \geq \Sigma_{i, i} \Paren{1-\Diag-15n \gamma} .
	\]
 	Since $\Diag \leq \frac{1}{\DEL}$, this implies that $1-\Diag -15n \gamma \geq 0$ for $n \geq 2$. Therefore, Equation~\eqref{eq:IH} can be upper bounded by,
	\begin{align}
	& \leq  \frac{\Abs{\tilde{\Sigma}_{i,i+1}} + \tau_{i-1} \Abs{\tilde{\Sigma}_{i-1,i+1}}}{\Abs{ \tilde{\Sigma}_{i,i} } - \Abs{\tilde{\Lambda}_{i-1, i}} \Abs{\tilde{\Sigma}_{i-1,i} }} & \nonumber \\
	& \leq \frac{ \tau_{i-1} \cdot \Abs{\tilde{\Sigma}_{i-1,i+1}} + \Abs{\tilde{\Sigma}_{i,i+1}}}
	{\Abs{\tilde{\Sigma}_{i,i}} - \tau_{i-1} \cdot \Abs{\tilde{\Sigma}_{i-1,i}}} & \text{(From Inductive Hypothesis)} \nonumber \\
	& \leq \frac{\tau_{i-1} \Diag + \Diag + \gamma (\tau_{i-1} + 1)}{1-\tau_{i-1} \Diag - \gamma(\tau_{i-1} + 1)} & \text{(From data properties)} \label{eq:thirdIn}
	\end{align}
 	We will now show that Equation~\eqref{eq:thirdIn} can be upper-bounded by	$\tau_{i-1} + 5\gamma$. This will complete the induction.
 
	Note that Equation~\eqref{eq:thirdIn} can be written as 
		\begin{equation}
			\label{eq:advInt1}
			= \frac{\Diag(\tau_{i-1}+1) + \gamma (\tau_{i-1} + 1)}{(1-\tau_{i-1} \Diag) \Brac{1-\frac{\gamma(\tau_{i-1}+1)}{1-\tau_{i-1}\Diag}}}	
		\end{equation}
		Recall that we have $\gamma \leq \tfrac{1}{\GAM}$, $\tau_{i-1} \leq 1+5n \gamma$ and $\alpha \leq \frac{1}{5}$. Consider $\frac{\gamma(\tau_{i-1}+1)}{1-\tau_{i-1}\Diag}$. This can be written as 
		\begin{align*}
				\frac{\gamma(\tau_{i-1}+1)}{1-\tau_{i-1}\Diag} & \leq \frac{\tfrac{1}{\GAM} \left[ 1 + \frac{5 n}{\GAM} + 1 \right]}{1 - \left( 1+\tfrac{5n}{\GAM} \cdot \tfrac{1}{5} \right)} \leq \frac{\tfrac{2}{n^6} + \tfrac{5}{n^5}}{\tfrac{4}{5} - \tfrac{1}{n^5}}  \leq \frac{\tfrac{2}{64} + \tfrac{5}{32}}{\tfrac{4}{5} - \tfrac{1}{32}}  \leq 0.25.
		\end{align*}
		Thus, we have $\Brac{1-\frac{\gamma(\tau_{i-1}+1)}{1-\tau_{i-1}\Diag}} \geq 0.75$. Therefore, Eq.~\eqref{eq:advInt1} can be upper-bounded by,
		
		\begin{align}
			& \leq \frac{4 \Diag(\tau_{i-1}+1) + 4\gamma (\tau_{i-1} + 1)}{3(1-\tau_{i-1} \Diag)} & \nonumber \\
			& =  \frac{4 \Diag(\tau_{i-1}+1)}{3(1-\tau_{i-1} \Diag)} +  \frac{4\gamma (\tau_{i-1} + 1)}{3(1-\tau_{i-1} \Diag)} \label{eq:advInt2}
		\end{align}
		Consider $ \frac{4\gamma (\tau_{i-1} + 1)}{3(1-\tau_{i-1} \Diag)}$. When $n \geq 2$ we have,
		\begin{align*}
			 \frac{4\gamma (\tau_{i-1} + 1)}{3(1-\tau_{i-1} \Diag)} & \leq \frac{4}{3} \left( \frac{1 + 5 n \gamma + 1}{1-(1 + 5 n \gamma)\tfrac{1}{5}} \right) \leq \frac{4}{3} \left( \frac{2 + \tfrac{5}{n^5}}{\tfrac{4}{5} - \tfrac{1}{n^5}} \right) \leq 	\frac{4}{3} \left( \frac{2 + \tfrac{5}{32}}{\tfrac{4}{5} - \tfrac{1}{32}} \right) \leq 5.
		\end{align*}
		Thus, Eq.~\eqref{eq:advInt2} can be upper-bounded by,
		\begin{align*}
			& \leq \frac{4 \Diag(\tau_{i-1}+1)}{3(1-\tau_{i-1} \Diag)} + 5 \gamma \leq \frac{\tfrac{4}{15} (1 + \tau_{i-1})}{3\left(1-\tfrac{\tau_{i-1}}{5} \right)} + 5 \gamma
		\end{align*}
		Thus, as long as $\frac{\tfrac{4}{15} (1 + \tau_{i-1})}{3\left(1-\tfrac{\tau_{i-1}}{5} \right)} \leq \tau_{i-1}$ the induction holds. Rearranging, this inequality holds if and only if
		\[
				9 \tau_{i-1}^2 - 41 \tau_{i-1} + 4 \leq 0.
		\]
		Thus the inequality holds as long as $\tau_{i-1} \in [0.099, 4.45]$. From the definition we have $\tau_{i-1} \geq \tau_1 = 1 + 5\gamma$. From the inductive hypothesis we have $\tau_{i-1} \leq 1 + 5 (i-1) \gamma$. Thus, $\tau_{i-1} \in [0.099, 4.45]$ holds.
		This completes the inductive step.

	The above induction implies that $\tau_1 \leq \tau_2 \leq \ldots \leq \tau_{n} \leq \tau + 5n \gamma$.
\end{proof}

Next, we show that the relative distance between the real parameter $\mathbf{\Lambda}$ and the recovered parameter from the perturbed instance $\mathbf{\tilde{\Lambda}}$ is not too large.
\begin{lemma}
	\label{lem:main}
	Let $\tau = 1 + 5n \gamma$. Define $\beta_c := \frac{[(3 + 3\tau) \Diag + (\tau + 1)]}{1- (\tau + 2) \Diag - (\tau + 1) \Diag^2 - 4n \gamma}$. Then for each $i \in [n-1]$ we have that,
	\[ \Abs{\mui{i}{i+1} - \tilde{\Lambda}_{i, i+1}} \leq \beta_c \cdot \frac{ \gamma}{1-\gamma} \mper \]
\end{lemma}

\begin{proof}
	We will prove this via induction on $i$. The base case is when $i=1$. Note that $\Lambda_{1, 2} = \frac{\Sigma_{1, 2}}{\Sigma_{1, 1}}$ and $\tilde{\Lambda}_{1, 2}=\frac{\tilde{\Sigma}_{1, 2}}{\tilde{\Sigma}_{1, 1}}$. The expression $\Abs{\mui{1}{2} - \tilde{\Lambda}_{1, 2}}$ evaluates to,
	 \[
	 \Abs{\frac{\Sigma_{1, 2}\Sigma_{1, 1} + \Sigma_{1, 2}\e_{1, 1} - \Sigma_{1, 2}\Sigma_{1, 1} - \Sigma_{1, 1} \e_{1, 2}}{\Sigma_{1, 1}\Sigma_{1, 1} + \e_{1, 1} \Sigma_{1, 1}}}.
	 \]
	This can be upper-bounded by,
	  \begin{align*}
	  		\frac{\Abs{\Sigma_{1, 2}\e_{1, 1} - \Sigma_{1, 1} \e_{1, 2}}}{\Abs{\Sigma_{1, 1}\Sigma_{1, 1} + \e_{1, 1} \Sigma_{1, 1}}} & \leq \frac{\Abs{\Sigma_{1, 2}} \Abs{\e_{1, 1}} + \Sigma_{1, 1} \Abs{\e_{1, 2}}}{\Sigma_{1, 1}^2 - \Abs{\e_{1, 1}} \Sigma_{1, 1}} \\
	  		 & \leq \frac{\Diag \gamma \Sigma_{1, 1}^2 + \gamma \Sigma_{1, 1}^2}{\Sigma_{1, 1}^2 - \gamma \Sigma_{1, 1}^2} \\
	  		& = \frac{\gamma(1+\Diag)}{1-\gamma }.
	  \end{align*}
		The second inequality uses the perturbation properties in Model~\ref{mod:localDominance}. Therefore, when $i=1$, we have that $\frac{\Abs{\mui{1}{2} - \tilde{\Lambda}_{1, 2}}}{\Abs{\Lambda_{1, 2}}} \leq \gamma/(1-\gamma) \left( 1+\Diag \right)$.
		
		Note that $\beta_c \geq 1+ \tau \geq 2$. Consider $\frac{1+\alpha}{1-\gamma}$. Since $alpha \leq \frac{1}{5}$ and $\gamma \leq \frac{1}{\GAM}$ we have  $\frac{1+\alpha}{1-\gamma} \leq \frac{1+\tfrac{1}{5}}{1-\tfrac{1}{32}} \leq 2 \leq \beta_c$. This completes the base case.
	
	We will now consider the case when $i \geq 2$ with the inductive hypothesis that for all $k=1, 2, \ldots, i-1$, the statement in the lemma is true. 
	
	Consider the expression $\Abs{\mui{i}{i+1} - \tilde{\Lambda}_{i, i+1}}$. Expanding out the terms based on Equation~\eqref{eqn:Recurrence} we get the absolute value of the following in the numerator.
	\begin{align*}
	&\Lambda_{i-1, i} \tilde{\Lambda}_{i-1, i} \Sigma_{i-1, i+1} \left(\Sigma_{i-1, i} + \e_{i-1, i} \right)  - \tilde{\Lambda}_{i-1, i} \Sigma_{i, i+1} \left( \Sigma_{i-1, i} + \e_{i-1, i} \right)\\
	& -\Lambda_{i-1, i} \Sigma_{i-1, i+1} \left( \Sigma_{i, i} + \e_{i, i} \right)  + \Sigma_{i, i+1} \left( \Sigma_{i, i} + \e_{i, i} \right) \\
	&- \Lambda_{i-1, i} \tilde{\Lambda}_{i-1, i} \Sigma_{i-1, i} \left( \Sigma_{i-1, i+1} + \e_{i-1, i+1} \right)  + \Lambda_{i-1, i} \Sigma_{i-1, i} \left( \Sigma_{i, i+1} + \e_{i, i+1} \right)\\
	&+\tilde{\Lambda}_{i-1, i} \Sigma_{i, i} \left( \Sigma_{i-1, i+1} + \e_{i-1, i+1} \right)  - \Sigma_{i, i} \left( \Sigma_{i, i+1} + \e_{i, i+1} \right).
	\end{align*}
	First consider the terms without $\e_{i, j}$ as a multiplicative factor in the above expression. After cancellations, they evaluate to
	\begin{align*}
	 & \Abs{-\tilde{\Lambda}_{i-1, i} \Sigma_{i, i+1} \Sigma_{i-1, i}-\Lambda_{i-1, i} \Sigma_{i-1, i+1}\Sigma_{i, i} + \Lambda_{i-1, i} \Sigma_{i-1, i}\Sigma_{i, i+1}+\tilde{\Lambda}_{i-1, i} \Sigma_{i, i}\Sigma_{i-1, i+1}} \\
	 & \leq \Abs{\Lambda_{i-1, i} - \tilde{\Lambda}_{i-1, i}} \Abs{\Sigma_{i-1, i}}\Abs{\Sigma_{i, i+1}} + \Abs{\Lambda_{i-1, i} - \tilde{\Lambda}_{i-1, i}}\Sigma_{i, i} \Abs{\Sigma_{i-1, i+1}}.
	\end{align*}
	Using the inductive hypothesis, we have that $\Abs{ \Lambda_{i-1, i}-\tilde{\Lambda}_{i-1, i}} \leq \beta_c \cdot \frac{\gamma}{1-\gamma}$. Therefore, we have,
	\begin{align*}
		& \leq \frac{\gamma}{1-\gamma} \cdot \beta_c \cdot \Diag^2 \cdot \Sigma_{i, i}^2 + \frac{\gamma}{1-\gamma} \cdot \beta_c \cdot \Diag \cdot \Sigma_{i, i}^2 \\
		& \leq \frac{\gamma}{1-\gamma} \cdot \beta_c \cdot (\Diag^2 + \Diag) \cdot \Sigma_{i, i}^2.
	\end{align*}
	Among the terms involving $\e_{i, j}$ consider those which are multiplied by a $\Sigma_{i, i}$.  The terms $\Sigma_{i, i} \e_{i, i+1}$ and $\tilde{\Lambda}_{i-1, i} \Sigma_{i, i}\e_{i-1, i+1}$ can be upper-bounded by $\gamma \Sigma^2_{i ,i}$ and $\tau \gamma \Sigma^2_{i ,i}$ respectively. Among the remaining $6$ terms, three of them, namely $\Abs{\Lambda_{i-1, i} \Sigma_{i-1, i+1} \e_{i, i}}$, $\Abs{\Sigma_{i, i+1} \e_{i, i}}$ and $\Abs{\Lambda_{i-1, i}\Sigma_{i-1, i} \e_{i, i+1}}$, can each be upper-bounded by $\gamma \Diag \Sigma^2_{i, i}$ and the remaining three, namely $\Abs{\Lambda_{i-1, i} \tilde{\Lambda}_{i-1, i} \Sigma_{i-1, i+1}\e_{i-1, i}}$, $\Abs{\tilde{\Lambda}_{i-1, i} \Sigma_{i, i+1}\e_{i-1, i}}$ and $\Abs{\Lambda_{i-1, i} \tilde{\Lambda}_{i-1, i} \Sigma_{i-1, i}\e_{i-1, i+1}}$, can each be upper-bounded by $\tau \gamma \Diag \Sigma^2_{i, i}$. Therefore, we have that the absolute value of the numerator can be upper-bounded by,
	 \[
	 		\frac{\gamma}{1-\gamma} \cdot \beta_c \cdot \left( \Diag^2 + \Diag \right) \cdot \Sigma_{i, i}^2 + 3\left( \tau + 1 \right) \gamma \Diag \Sigma_{i, i}^2 + \left( \tau + 1 \right) \gamma \Sigma_{i, i}^2.
	 \]
	Since $\gamma \geq 0$, we have that $\gamma \leq \frac{\gamma}{1-\gamma}$. Therefore, the numerator can be upper-bounded by
	\begin{equation}
		\label{eq:IndNum}
			\frac{\gamma}{1-\gamma}  \left( \beta_c \Diag^2 + (\beta_c+ 3(\tau + 1)) \Diag + (\tau + 1)  \right)\Sigma_{i, i}^2.
	\end{equation}
	We want to lower-bound the absolute value of the denominator, which can be written as, 
	\begin{align*}
		& \Lambda_{i-1, i} \tilde{\Lambda}_{i-1, i} \Sigma_{i-1, i} \left( \Sigma_{i-1, i} + \e_{i-1, i} \right) - \Lambda_{i-1, i} \Sigma_{i-1, i} \left( \Sigma_{i, i} + \e_{i, i} \right) \\
		&-\tilde{\Lambda}_{i-1, i} \Sigma_{i, i} \left( \Sigma_{i-1, i} + \e_{i-1, i} \right) + \Sigma_{i, i} \left( \Sigma_{i, i} + \e_{i, i} \right).	
	\end{align*}
	Using the model properties, we can lower-bound the above expression by,
	\begin{align*}
		& \Sigma_{i, i}^2 - \gamma \Sigma_{i, i}^2 - \tau \gamma \Sigma_{i, i}^2 -
		\tau \Diag \Sigma_{i, i}^2 - \Diag \gamma \Sigma_{i, i}^2 - \Diag \Sigma_{i, i}^2 - \tau \gamma \Diag \Sigma_{i, i}^2 - \tau \alpha^2 \Sigma_{i, i}^2.
	\end{align*}
	Since $\alpha \leq \frac{1}{5}$, this can be lower-bounded by,
	\begin{equation}
		\label{eq:IndDen}	
		\Abs{1- (\tau + 1) \Diag - \tau \Diag^2 -2 (1+ \tau) \gamma} \Sigma^2_{i, i}.
	\end{equation}
	
	First, we show that $1- (\tau + 1) \Diag - \tau \Diag^2 -2 (1+ \tau) \gamma > 0$. Recall that $\tau= 1 + 5n \gamma$. Re-arranging we want to show that $\tau < \frac{1- 2 \gamma - \alpha}{\alpha^2 + \alpha + 2 \gamma}$ holds. Note that $\tau = 1 + 5n \gamma \leq 2$ when $\gamma \leq \frac{1}{\GAM}$ and $n \geq 2$. Consider $\frac{1- 2 \gamma - \alpha}{\alpha^2 + \alpha + 2 \gamma}$. This is at least $\frac{1 - \tfrac{1}{2^5} - \tfrac{1}{5}}{\tfrac{1}{25} + \tfrac{1}{5} + \tfrac{1}{2^5}}$ when $\gamma \leq \frac{1}{\GAM}, n \geq 2$ and $\alpha \leq \frac{1}{5}$. The RHS evaluates to $2.83$ which is larger than $2$ and thus larger than $\tau$.
	
	Thus, to complete the equation, combing Eq.~\eqref{eq:IndNum}, \eqref{eq:IndDen},  we have to show that,
	 \[
	 	\frac{\frac{\gamma}{1-\gamma}  \left( \beta_c \Diag^2 + (\beta_c+ (3+3\tau)) \Diag + (\tau+1) \right)}{1- (\tau + 1) \Diag - \tau \Diag^2 - 2 (1+ \tau) \gamma} \leq \beta_c \cdot \frac{\gamma}{1-\gamma}.
	 \]
	Re-arranging, $\frac{\left( \beta_c \Diag^2 + (\beta_c+ (3+3\tau)) \Diag + (\tau+1) \right)}{1- (\tau + 1) \Diag - \tau \Diag^2 - 2 (1+ \tau) \gamma} \leq \beta_c$ if and only if,
	 \[
	 	\beta_c \geq \frac{[(3 + 3\tau) \Diag + (\tau + 1)]}{1- (\tau + 2) \Diag - (\tau + 1) \Diag^2 - 2 (1+ \tau) \gamma}.
	 \]
	 Thus, for $\beta_c > 0$ we want,
	  \begin{equation}
	 	\label{eq:claimRemark}	
	 	 \left[ 1- (\tau + 2) \Diag - (\tau + 1) \Diag^2 - 2 (1+ \tau) \gamma \right] > 0.
	 \end{equation}
	  Rearranging Equation~\eqref{eq:claimRemark} we want  $\tau < \frac{1 - 2 \alpha - \alpha^2 - 2 \gamma}{\alpha + \alpha^2 + 2 \gamma}$. When $\alpha \leq \frac{1}{5}, \gamma \leq \frac{1}{\GAM}, n \geq 2$, the RHS is at least $\frac{1-\tfrac{2}{5}-\tfrac{1}{25}-\tfrac{1}{32}}{\tfrac{1}{5} + \tfrac{1}{25} + \tfrac{1}{32}} \geq 1.94$. On the other hand, $\tau = 1 + 5 n \gamma \leq 1+ \frac{5}{n^5} \leq 1.16$. Thus Eq.~\eqref{eq:claimRemark} holds and we have that $\beta_c$ is a constant greater than $0$. Note that $\beta_c \leq O(1)$ where the maximum is achieved when $\Diag = \frac{1}{\DEL}$.
	
	This completes the inductive step. Therefore, from induction the theorem holds for all $i \in [n-1]$.
\end{proof}

\subsection{Proof of Theorem~\ref{thm:main2}}
	We are now ready to prove the main Theorem~\ref{thm:main2}. From Lemma~\ref{lem:main} and the model assumptions we have $\Rel{\mathbf{\Lambda}}{\tilde{\mathbf{\Lambda}}} = \frac{\Abs{\mui{i}{i+1} - \tilde{\Lambda}_{i, i+1}}}{\Abs{\Lambda_{i, i+1}}} \leq \beta_c \cdot \frac{ \gamma}{1-\gamma} \cdot \lambda$. From perturbation properties in the Model~\ref{mod:localDominance}, and the definition of $i^*, j^*$, we have $\Rel{\mathbf{\Sigma}}{ \tilde{\mathbf{\Sigma}}} = \frac{\gamma \Abs{\Sigma_{i^*, j^*}}}{\Abs{\Sigma_{i^*, j^*}}} = \gamma$.
	Therefore, 
	\[
		\frac{\Rel{\mathbf{\Lambda}}{\tilde{\mathbf{\Lambda}}}}{\Rel{\mathbf{\Sigma}}{ \tilde{\mathbf{\Sigma}}}} \leq \frac{\lambda \cdot \beta_c}{1-\gamma}.
	\] 
	This implies that,
	 \[
	 		\lim_{\gamma \rightarrow 0^+} \frac{\Rel{\mathbf{\Lambda}}{\tilde{\mathbf{\Lambda}}}}{\Rel{\mathbf{\Sigma}}{ \tilde{\mathbf{\Sigma}}}} \leq \lambda \beta_c \leq O(\lambda) \leq O(n^{2}).
	 \]
	The second last inequality used the fact that $\beta_c \leq O(1)$ and the last inequality used the fact that $\frac{1}{\lambda} \geq \Omega(n^{-2})$.

\section{Instability Example}

In this section, we show that there exist simple \emph{bow-free path} examples where the condition number can be arbitrarily large.  We consider a point on the measure $0$ set of unidentifiable parameters and apply a small perturbation (of magnitude $\epsilon$).  This instance in the parameter space is identifiable (by using \cite{FDD2012Annals}). We then apply a perturbation (of magnitude $\gamma$) to this instance and compute the condition number. By making $\epsilon$ arbitrarily close to $0$, we obtain as large a condition number as desired. Any point on the singular variety could be used for this purpose and the proof of Theorem~\ref{thm:mainIns} provides a concrete instance. The example we construct is as follows. 
Consider a path of four vertices. Fix a small value $\epsilon$. Define the matrices $\vec{\Omega}$ and $\vec{\Lambda}$ as follows.
\[ \vec{\Omega} = 
\begin{bmatrix}
1 & 0 & 1/2 & 1/2 \\
0 & 1 & 0 & 1/2 \\  
1/2 & 0 & 1 + \epsilon & 0 \\
1/2 & 1/2 & 0 & 1  
\end{bmatrix}  \succeq 0 \]

and $\Lambda_{1,2} = \sqrt{2},  \Lambda_{2,3} = - \sqrt{2}, \Lambda_{3, 4} = \frac{1}{2}$.

Thus, we have the following.
\[ (\mathds{1}-\vec{\Lambda})^{-1} = 
\begin{bmatrix}
1 & \sqrt{2} & -2 & -1 \\
0 & 1 & -\sqrt{2} & -1/\sqrt{2} \\  
0 & 0 & 1 & 1/2 \\
0 & 0 & 0 & 1  
\end{bmatrix}\]  
Multiplying with $\vec{\Omega}$ we get $ (\mathds{1}-\vec{\Lambda})^{-T} \vec{\Omega}$ is,
\[\begin{bmatrix}
1 & 0 & 1/2 & 1/2 \\
\sqrt{2} & 1 & 1/\sqrt{2} & 1/\sqrt{2} + 1/2 \\  
-3/2 & -\sqrt{2} &  \epsilon & -1-1/\sqrt{2} \\
-1/4 & -1/\sqrt{2}+1/2 & \epsilon/2 & 1/2(1-1/\sqrt{2}) 
\end{bmatrix}\]  
Therefore, the resulting matrix $\vec{\Sigma}$, obtained by $(\mathds{1}-\vec{\Lambda})^{-T} \vec{\Omega} (\mathds{1}-\vec{\Lambda})^{-1}$ is,
\[ \textstyle
\begin{bmatrix}
1 & \sqrt{2} & -\tfrac{3}{2} & -\tfrac{1}{4} \\
\sqrt{2} & 3 & -\tfrac{5}{\sqrt{2}} & \tfrac{1}{2} - \tfrac{3}{2\sqrt{2}} \\  
-\tfrac{3}{2} & -\tfrac{5}{\sqrt{2}} & 5+\epsilon & \tfrac{3}{2} - \tfrac{1}{\sqrt{2}} + \tfrac{\epsilon}{2} \\
-\tfrac{1}{4} & \tfrac{1}{2} - \tfrac{3}{2 \sqrt{2}} & \tfrac{3}{2} - \tfrac{1}{\sqrt{2}} + \tfrac{\epsilon}{2} & \tfrac{5}{4}-\tfrac{1}{\sqrt{2}}  + \tfrac{\epsilon}{4}
\end{bmatrix} \]

Perturb every entry of $\vec{\Sigma}$ additively by $\gamma$ to obtain $\tilde{\vec{\Sigma}}$. This will ensure that $\Rel{\vec{\Sigma}}{\tilde{\vec{\Sigma}}} = 4\gamma$ since all entries in $\vec{\Sigma}$ are at least $1/4$. Now we show that in the reconstruction of $\tilde{\vec{\Lambda}}$ from $\vec{\Sigma}_{\gamma}$, the entry $\tilde{\vec{\Lambda}}_{3, 4} = 1$. This implies that the condition number $\kappa(\vec{\Lambda}, \vec{\Sigma}) = O\left( \frac{1}{\gamma} \right)$ and since $\gamma$ can be made arbitrarily close to $0$, this implies that the condition number is unbounded.

The denominator in the expression for $\tilde{\Lambda}_{3,4}$ in \eqref{eqn:Recurrence} is,  $-\Lambda_{2,3} \tilde{\Sigma}_{2,3} + \tilde{\Sigma}_{3,3}  =  \epsilon + \left( 1+\sqrt{2} \right) \gamma.$

Likewise the numerator in the expression for $\tilde{\Lambda}_{3, 4}$ evaluates to, $-\Lambda_{2,3} \tilde{\Sigma}_{2,4} + \tilde{\Sigma}_{3,4}  =  \epsilon/2 + \left( 1+\sqrt{2} \right) \gamma.$

Therefore when $\epsilon \rightarrow 0$ we have $\tilde{\Lambda}_{3, 4} \rightarrow 1$ and hence $\Rel{\vec{\Lambda}}{\vec{\tilde{\Lambda}}} = O(1) \neq 0$.

\section{When do random instances satisfy model assumptions?}
	\label{sec:random}
In this section, we prove theoretically that randomly generated $\mathbf{\Lambda}$ and $\mathbf{\Omega}$ satisfy the Model~\ref{mod:localDominance} for a natural generative process, albeit with slightly weaker constants. 

\xhdr{Generative model for $\LSEM$ instances.} 
We consider the following generative model for the $\LSEM$ instances. $\mathbf{\Lambda} \in \mathbb{R}^{n \times n}$ is generated by choosing each non-zero entry to be a sample from the uniform $\mathcal{U}[-h, h]$ distribution. $\mathbf{\Omega} \in \mathbb{R}^{n \times n}$ is chosen by first sampling $n$-dimensional vectors $\mathbf{v}_1, \mathbf{v}_2, \ldots, \mathbf{v}_n \in \mathbb{R}^d$ from a $d$-dimensional unit sphere, such that $\mathbf{v}_i$ is a uniform sample in the subspace perpendicular to $\mathbf{v}_{i-1}$ for every $i$ (\emph{i.e.,} $\langle \mathbf{v}_i, \mathbf{v}_{i-1} \rangle =0$) and then letting $\Omega_{i, j} = \langle \vec{v}_i, \vec{v}_j \rangle$. Note, in the scenario when $\vec{\Omega}$ need not follow a specified \emph{zero-pattern} then first generating vectors $\vec{v}_1, \vec{v}_2, \ldots, \vec{v}_n$ independently and uniformly from a $d$-dimensional unit sphere and then setting $\Omega_{i, j} = \langle \vec{v}_i, \vec{v}_j \rangle$ gives an uniform distribution over PSD matrices. Hence, the above generative procedure is a natural extension to randomly generating PSD matrices with a specified zero-patterns.

\subsection{Regime when the generative model satisfy the Model~\ref{mod:localDominance} properties} 
	
	First, we prove that the random generative process satisfies the bound $\Diag \leq \frac{1}{\DEL}$ in Theorem~\ref{thm:converge} below.

	\begin{theorem}
	\label{thm:converge}
	Consider $\mathbf{\Lambda}$ and $\mathbf{\Omega}$ generated using the above random model with $0 \leq h < \frac{1}{\sqrt{2}}$. Then there exists a $0 \leq \zeta_h \leq h + o(1)$ and $\Omega(n^{-2}) \leq \chi_h$, such that for every sufficiently large $n$, there exists $0 \leq \delta_{1, n} < \frac{1}{2}$ and $0 \leq \delta_{2, n} < \frac{1}{2}$ such that,
	\begin{equation}
		\label{eq:negModel}
		\mathbb{P}\left[ \forall~i~\Abs{\Sigma_{i-1, i}}, \Abs{\Sigma_{i, i+1}}, \Abs{\Sigma_{i-1, i+1}} \leq \zeta_h \Sigma_{i,i} \right] \geq 1-\delta_{1, n}.
	\end{equation}
	\begin{equation}
		\label{eq:negModellower}
		\mathbb{P}\left[ \forall~i~\Abs{\Lambda_{i, i-1}} \geq \chi_h \right] \geq 1-\delta_{2, n}.
	\end{equation}
Moreover, we have that $\lim_{n \rightarrow \infty} \delta_{1, n} = 0$ and $\lim_{n \rightarrow \infty} \delta_{2, n} = 0$. Thus, the above statements hold with high-probability.
\end{theorem}

\begin{remark}
	Note that in this theorem, the constant $\zeta_h$ for $\Diag$ is at most $h$. When $0 \leq h < \frac{1}{\sqrt{2}}$, this is at most $\frac{1}{\sqrt{2}}$ where the maximum is achieved when $h=\frac{1}{\sqrt{2}}$. Moreover, when $h < 0.2$, we have $h \leq \frac{1}{\DEL}$. Therefore, this satisfies the \emph{exact} requirement in data properties of Model~\ref{mod:localDominance} when $h < 0.2$ and in the regime $0.2 \leq h < \frac{1}{\sqrt{2}}$, it satisfies the property for a slightly larger value of $\Diag$, namely, $\Diag \leq \frac{1}{\sqrt{2}}$. Nonetheless in Section~\ref{sec:experiments} we show experimentally that even in this regime the instance is well-conditioned.
\end{remark}

\xhdr{Notation.} Throughout this section, we use the notation $F(d) \defeq \exp \left[ - c_1 d^{c_2} \right]$ and $\mathcal{I}_d \defeq \left[ -\frac{C_{\text{conc}}}{d^{0.25}}, \frac{C_{\text{conc}}}{d^{0.25}} \right]$ for some absolute constants $C_{\text{conc}}, c_1, c_2 > 0$.

	\xhdr{Properties of the generated $\vec{\Lambda}$ and $\vec{\Omega}$.} We state a few useful properties of the above generative process in the following lemmas. We defer their proofs to Subsection~\ref{subsec:missing}.
	
	\begin{lemma}
		\label{lem:orthoSphere}
		Let $\mathbf{v}_1, \mathbf{v}_2, \ldots, \mathbf{v}_n \in \mathbb{R}^d$ be $n$ vectors drawn independently from a $d$-dimensional unit sphere, such that $\vec{v}_i$ is drawn uniformly from the subspace perpendicular to $\mathbf{v}_{i-1}$ for every $i$. Then for $i, j \in [n]$ such that $|j-i| \geq 2$ we have with probability at least $1-F(d)$, $\langle \mathbf{v}_i, \mathbf{v}_j \rangle \in \mathcal{I}_d$.
	\end{lemma}
	
	\begin{corollary}
		\label{cor:Omega}
		With probability at least $1-n^2F(d)$, $\mathbf{\Omega} \in \mathbb{R}^{n \times n}$ is a matrix with $\Omega_{i, i}=1$ for all $i \in [n]$ and $\Omega_{i, j} \in \mathcal{I}_d$ when $|i-j| \geq 2$.
	\end{corollary}

	\begin{lemma}
		\label{lem:boundLambda}
		Let $\sigma_h^2 < 1/2$ denote the variance of a $\mathcal{U}[-h, h]$ random variable. For any $i \leq j$ such that $j-i \geq d^{0.1}$ where $d \geq 4^{10} \log^{10}(n)$, we have that,
		\[
				\Pr{ \Abs{\Lambda_{i, j}} \geq \sigma_h^{d^{0.1}/2}} \leq \frac{1}{n^4}
		\]	
	\end{lemma}

	\begin{corollary}
		\label{cor:Lambda}
		With probability at least $1-\frac{1}{n^2}$, we have that for every pair $(i, j)$ such that $j-i \geq d^{0.1}$ the inequality $\Abs{\Lambda_{i, j}} \leq \sigma_h^{d^{0.1}/2}$ holds. Moreover, for the remaining pairs $(i, j)$ we have the trivial upper-bound of $\Abs{\Lambda_{i, j}} \leq 1$.
	\end{corollary}

	The corollary follows from Lemma~\ref{lem:boundLambda} and using an union bound over $O(n^2)$ pairs of $i, j$. 

\begin{proof}[Proof of Theorem~\ref{thm:converge}]
	Using Equation~\eqref{eq:sigmaEq}, we have that for any $i, j$ the term $\Sigma_{i, j}$ can be written as,
	\begin{equation}
		\label{eq:SigmaExpansion}
		\Sigma_{i, j} = \sum_{k=1}^i \sum_{k'=1}^j \Lambda_{k, i} \Lambda_{k', j} \Omega_{k, k'}.	
	\end{equation}
	More specifically using the Taylor series expansion we get,
	\begin{equation}
		\label{eq:pathek}
		\paren{\mathbf{I} - \mathbf{\Lambda}}^{-1} \mathbf{e}_k = \paren{\mathbf{I} + \mathbf{\Lambda} + \ldots + \mathbf{\Lambda}^{n-1}} \mathbf{e}_k 
		= \sum_{\ell = 1}^k \paren{\Pi_{j = \ell}^{k-1} \Lambda_{j,j+1}} \mathbf{e}_\ell
		= \sum_{\ell = 1}^k \Lambda_{\ell ,k} \mathbf{e}_\ell \mcom 
\end{equation}
where $\Lambda_{\ell, k}$ is as defined in Equation~\eqref{eq:Lambdaij}.

Thus, $\Sigma_{i, j} = \left( \sum_{\ell = 1}^i \Lambda_{\ell , i} \mathbf{e}_\ell^T \right)	 \vec{\Omega} \left( \sum_{\ell' = 1}^j \Lambda_{\ell' , j} \mathbf{e}_{\ell'} \right)$ which evaluates to the expression in Equation~\eqref{eq:SigmaExpansion}.

	\xhdr{Conditioning on high-probability events.} We will condition on the event that $\vec{\Omega}$ satisfies the properties in Corollary~\ref{cor:Omega} and $\vec{\Lambda}$ satisfies the properties in Corollary~\ref{cor:Lambda}. Taking a union bound, these events happen with probability at least $1-n^{2}F(d)-n^{-2} = 1-\frac{1}{\poly (n) }$. Thus, in what follows, the matrices $\vec{\Omega}$ and $\vec{\Lambda}$ are deterministic and satisfies properties in Corollary~\ref{cor:Omega} and \ref{cor:Lambda} respectively. We call this the \emph{good} event.
	
	As long as we have the good event, from Equation~\eqref{eq:SigmaExpansion} we have the following.
	\begin{equation}
		\label{eq:i1i}
		\Abs{\Sigma_{i-1, i}} \leq \sum_{k=1}^{i-1} \Abs{\Lambda_{k, i-1} \Lambda_{k, i}} + \sum_{k =1}^{i-1} \sum_{k' \neq k: k'=1}^{i} C_{\text{conc}}\frac{\Abs{\Lambda_{k, i-1} \Lambda_{k', i}}}{d^{0.25}},
	\end{equation}
	
	\begin{equation}
		\label{eq:ii1}
		\Abs{\Sigma_{i, i+1}} \leq \sum_{k=1}^{i} \Abs{\Lambda_{k, i} \Lambda_{k, i+1}} + \sum_{k=1}^{i} \sum_{k \neq k':k'=1}^{i+1} C_{\text{conc}} \frac{\Abs{\Lambda_{k, i} \Lambda_{k', i+1}}}{d^{0.25}},
	\end{equation}
	
	\begin{equation}
		\label{eq:i1i1}
		\Abs{\Sigma_{i-1, i+1}} \leq \sum_{k=1}^{i-1} \Abs{\Lambda_{k, i-1} \Lambda_{k, i+1}} + \sum_{k=1}^{i-1} \sum_{k \neq k': k'=1}^{i+1} C_{\text{conc}} \frac{\Abs{\Lambda_{k, i-1} \Lambda_{k', i+1}}}{d^{0.25}},
	\end{equation}
	
	\begin{equation}
		\label{eq:ii}
		\Sigma_{i, i} \geq \sum_{k=1}^{i} \Lambda^2_{k, i} - \sum_{k=1}^{i} \sum_{k \neq k': k'=1}^{i} C_{\text{conc}} \frac{\Abs{\Lambda_{k, i} \Lambda_{k', i}}}{d^{0.25}}.
	\end{equation}
	More specifically, we obtain the above equations by starting with Equation~\eqref{eq:SigmaExpansion} and grouping the terms where $k=k'$ to obtain the first summand and grouping the remaining terms to obtain the second summand. Moreover, using Corollary~\ref{cor:Omega} and the good event, in the first summand we have $\Omega_{k, k} =1$ and in the second summand we have $-\frac{C_{\text{conc}}}{d^{0.25}} \leq \Omega_{k, k'} \leq \frac{C_{\text{conc}}}{d^{0.25}}$. 
	
	We first bound the terms in the second summand of Equations~\eqref{eq:i1i}, \eqref{eq:ii1}, \eqref{eq:i1i1}, \eqref{eq:ii}. Define $\mG(n) := \tfrac{C_{\text{conc}}}{d^{0.05}} + \tfrac{C_{\text{conc}}}{n^2 \log^{10/4} n}$. Then, we have the following lemma.
	
	\begin{lemma}
		\label{lem:smallTermsBound}
		As long as the good event holds, we have the following for every $j, j' \in [n]$.
		\begin{equation}
			\label{eq:MainSmallTerm}
			\sum_{k =1}^{j'} \sum_{k' \neq k: k'=1}^{j} C_{\text{conc}}\frac{\Abs{\Lambda_{k, j'} \Lambda_{k', j}}}{d^{0.25}} \leq \mG(n).
		\end{equation}
		Thus, as a corollary we get the following for every $i \in [n]$ by substituting the appropriate values for $j, j'$.
		\begin{align*}
				\sum_{k =1}^{i-1} \sum_{k' \neq k: k'=1}^{i} C_{\text{conc}}\frac{\Abs{\Lambda_{k, i-1} \Lambda_{k', i}}}{d^{0.25}} \leq & \mG(n), \\
				\sum_{k=1}^{i} \sum_{k \neq k':k'=1}^{i+1} C_{\text{conc}} \frac{\Abs{\Lambda_{k, i} \Lambda_{k', i+1}}}{d^{0.25}} \leq & \mG(n), \\
				\sum_{k=1}^{i-1} \sum_{k \neq k': k'=1}^{i+1} C_{\text{conc}} \frac{\Abs{\Lambda_{k, i-1} \Lambda_{k', i+1}}}{d^{0.25}} \leq & \mG(n), \\
				 \sum_{k=1}^{i} \sum_{k \neq k': k'=1}^{i} C_{\text{conc}} \frac{\Abs{\Lambda_{k, i} \Lambda_{k', i}}}{d^{0.25}} \leq & \mG(n).
		\end{align*}
	\end{lemma}
	
	Next, to bound the terms in the first summand of Equations~\eqref{eq:i1i}, \eqref{eq:ii1}, \eqref{eq:i1i1}, \eqref{eq:ii}, we will use the following lemma. Define $P_i := \sum_{k=1}^{i} \Lambda^2_{k, i}$. We have the following lemma that upper-bounds $P_i$.
	\begin{lemma}
		\label{obs:alphai}
		For every $i \geq 1$, the term $P_i \leq \frac{1}{1-h^2}$.
	\end{lemma}

	We defer the proofs of Lemmas~\ref{lem:smallTermsBound} and \ref{obs:alphai} to subsection~\ref{subsec:missing}. We will now prove that the following statement holds: for every $i \in [n]$ we have $\Abs{\Sigma_{i-1, i}} - \zeta_h \Sigma_{i, i} \leq 0$, $\Abs{\Sigma_{i, i+1}} - \zeta_h \Sigma_{i, i} \leq 0$ and $\Abs{\Sigma_{i-1, i+1}} - \zeta_h \Sigma_{i, i} \leq 0$. This completes the proof of Theorem~\ref{thm:converge}.

	Consider $\Abs{\Sigma_{i-1, i}} - \zeta_h \Sigma_{i, i}$. Using Equations~\eqref{eq:i1i} and \eqref{eq:ii} this can be upper-bounded by, 
	\[
			\sum_{k=1}^{i-1} \Lambda^2_{k, i-1} \Abs{\Lambda_{i-1, i}} - \zeta_h \sum_{k=1}^{i-1} \Lambda^2_{k, i-1} \Lambda^2_{i-1, i} - \zeta_h + 2 \mG(n).
	\]
	We now want that this expression has to be at most $0$ in Equation~\eqref{eq:negModel}. Thus, re-arranging we want,
	\[
			\zeta_h \geq \frac{P_{i-1} \Abs{\Lambda_{i-1, i}} + 2 \mG(n)}{1+P_{i-1} \Lambda^2_{i-1, i}}.
	\]
	First, we will show that the function in the RHS is increasing in $P_{i-1}$. Differentiating with respect to $P_{i-1}$ we get, 
	\[
			\frac{\left(1+P_{i-1} \Lambda^2_{i-1, i} \right) \Abs{\Lambda_{i-1, i}} - \left( \Abs{\Lambda_{i-1, i}} P_{i-1} + 2 \mG(n) \right) \Lambda^2_{i-1, i}}{\left( 1+P_{i-1} \Lambda^2_{i-1, i} \right)^2}.
	\]
	Note, the numerator evaluates to $\Abs{\Lambda_{i-1, i}} -  2 \mG(n) \Lambda^2_{i-1, i}$ which is always greater than $0$ when $ d \geq \Omega(\poly\log(n))$. From Lemma~\ref{obs:alphai}, we have $P_{i-1} := \sum_{k=1}^{i-1} \Lambda^2_{k, i-1} \leq \frac{1}{1-h^2}$. Thus, we want,
	\begin{equation}
		\label{eq:zeta1}
			\zeta_h \geq \frac{\Abs{\Lambda_{i-1, i}} + 2 \mG(n)}{1-h^2 + \Lambda^2_{i-1, i}}.
	\end{equation}
	Differentiating the RHS with respect to $\Abs{\Lambda_{i-1, i}}$ we get,
	\[
			\frac{\left( 1-h^2 + \Lambda^2_{i-1, i} \right) - 2 \Lambda^2_{i-1, i} - 4 \mG(n) \Abs{\Lambda_{i-1, i}} }{\left( 1-h^2 + \Lambda^2_{i-1, i} \right)^2}.
	\]
	Thus, as long as $h \leq \frac{1}{\sqrt{2}}$ and $d \geq \Omega(\poly\log(n))$ this value is non-negative and thus, the RHS in Equation~\eqref{eq:zeta1} is increasing. Therefore, we get the maximum value at $\Abs{\Lambda_{i-1, i}} = h$ with the value being $\zeta_h \geq h + 2 \mG(n)$.
	
	Similarly, consider $\Abs{\Sigma_{i, i+1}} - \zeta_h \Sigma_{i, i}$. Using Equations~\eqref{eq:ii1} and \eqref{eq:ii} this can be upper-bounded by,
	\[
			\sum_{k=1}^{i} \Lambda^2_{k, i} \Abs{\Lambda_{i, i+1}} - \zeta_h \sum_{k=1}^{i} \Lambda^2_{k, i} + 2 \mG(n).
	\]
	Re-arranging we get that,
	\[
			\zeta_h \geq \Abs{\Lambda_{i, i+1}} + 2 \mG(n) \geq h + 2 \mG(n).
	\]
	Since this is an increasing function in $\Abs{\Lambda_{i, i+1}}$, we get the last inequality.
	
	Finally, consider $\Abs{\Sigma_{i-1, i+1}} - \zeta_h \Sigma_{i, i}$. Using Equations~\eqref{eq:i1i1} and \eqref{eq:ii} this can be upper-bounded by,
	\[
			\sum_{k=1}^{i-1} \Lambda^2_{k, i-1} \Abs{\Lambda_{i-1, i+1}} - \zeta_h \sum_{k=1}^{i-1} \Lambda^2_{k, i-1} \Lambda^2_{i-1, i} - \zeta_h + 2 \mG(n).
	\]
	As in the previous two cases, solving for the inequality that this expression has to be at most $0$ in Equation~\eqref{eq:negModel}, re-arranging and using Observation~\ref{obs:alphai} we can show that the maximum is achieved when $P_{i-1} = \frac{1}{1-h^2}$. Thus, we get,
	 \[
	 		\zeta_h \geq \frac{\Abs{\Lambda_{i-1, i}\Lambda_{i, i+1}} + 2 \mG(n)}{1+\Lambda^2_{i-1, i} - h^2} \geq h^2 + 2 \mG(n).
	 \]
	Therefore, as long as $\zeta_h \geq \max \left\{h^2, h \right\} + 2 \mG(n)$ we have that the three equations above are all upper-bounded by $0$. Since $h < 1$ setting $\zeta_h = h + 2 \mG(n) \leq \frac{1}{\sqrt{2}}$ suffices. Finally, note that $\mG(n) = o(1)$.

	\xhdr{Proof of Equation~\eqref{eq:negModellower} in Theorem~\ref{thm:converge}.} We will show that the condition on $\Abs{\Lambda_{i, i+1}}$ in Model~\ref{mod:localDominance} holds for the random instances. From Lemma~\ref{appx:uniform}, we have that for a given $i \in [n-1]$, we have $\Pr{\Abs{\Lambda_{i, i+1}} \geq n^{-2}} \geq 1-n^{-2}$. Thus, for $n$ independent random variables we have, $\Pr{\forall i \in [n-1] \quad \Abs{\Lambda_{i, i+1}} \geq n^{-2}} \geq \left( 1-n^{-2} \right)^n$. Note that, $\lim_{n \rightarrow \infty} \left( 1-n^{-2} \right)^n = 1$. Therefore, this event happens with high-probability. 
\end{proof}

\subsubsection{Missing Proofs of Helper Lemmas.} 
	\label{subsec:missing}
	\begin{proof}[Proof of Lemma~\ref{lem:smallTermsBound}]
		We prove the main Equation~\eqref{eq:MainSmallTerm}. Fix an arbitrary $j, j' \in [n]$. Let $\mathcal{S}(j) := \{k: 1 \leq k \leq j, |k-j| \leq d^{0.1}\}$. $\mathcal{T}(j,  j') :=  \{(k, k'): k \neq k', 1 \leq k \leq  j', 1 \leq k' \leq  j\} \setminus \Brac{\mathcal{S}(j') \times \mathcal{S}(j)}$.
	\begin{equation}
		\label{eq:secondSplit}
		\sum_{k \in \mathcal{S}(j')} \sum_{k \neq k': k' \in \mathcal{S}(j)} C_{\text{conc}} \frac{\Abs{\Lambda_{k, j'} \Lambda_{k', j}}}{d^{0.25}} + \sum_{(k, k') \in \mathcal{T}(j, j')} C_{\text{conc}} \frac{\Abs{\Lambda_{k, j'} \Lambda_{k', j}}}{d^{0.25}}. 
	\end{equation}
	From Corollary~\ref{cor:Lambda}, we have $\Abs{\Lambda_{i, j}} \leq 1$ for every pair $(i, j)$ where $|i-j| \leq d^{0.1}$. Thus, we can upper-bound every term in the first summand of Equation~\eqref{eq:secondSplit} by $\frac{C_{\text{conc}} }{d^{0.25}}$. Moreover, the total number of terms in that summand is at most $d^{0.1} \times d^{0.1}$. Thus, the total sum is at most $\frac{C_{\text{conc}}}{d^{0.05}}$. 
	
	Consider the second summand in Equation~\eqref{eq:secondSplit}. From Corollary~\ref{cor:Lambda} we have that each summand can be upper-bounded by $\sigma_h^{d^{0.1}}$. Recall that $\sigma_h^2$ and thus, $\sigma_h$ are both strictly less than $1/2$ (Lemma~\ref{appx:uniform}). This implies that the \emph{sum} of the terms in the second summand is at most $\frac{C_{\text{conc}} n^2}{d^{0.25} \sigma_h^{-d^{0.1}}}$. Since $d \geq 4^{10} \log^{10}(n)$, this evaluates to $\frac{C_{\text{conc}}}{n^2 \log^{10/4} n}$. Therefore, the overall expression is upper-bounded by $\frac{C_{\text{conc}}}{d^{0.05}} + \frac{C_{\text{conc}}}{n^2 \log^{10/4} n} = \mG(n)$.
	\end{proof}
	
	\begin{proof}[Proof of Lemma~\ref{obs:alphai}]
		We will prove this by induction. The base case is $i=1$ where we have $P_1 = 1 \leq \frac{1}{1-h^2}$ when $h \geq 0$. For every $i \geq 2$, we have that $P_i = \Lambda^2_{i, i} + \Lambda^2_{i-1, i} P_{i-1}$. Therefore, $P_i \leq 1 + h^2 P_{i-1}$. From inductive hypothesis we have that $P_{i-1} \leq \frac{1}{1-h^2}$. Therefore, $P_i \leq 1 + \frac{h^2}{1-h^2} = \frac{1}{1-h^2}$. 
	\end{proof}

	\begin{proof}[Proof of Lemma~\ref{lem:orthoSphere}]

		The proof of this uses the standard concentration for vectors sampled independently from the $d$-dimensional unit sphere (see Lemma~\ref{lem:gaussianprojections}). We modify this to our setting as follows. Generate $n$ vectors $\mathbf{\tilde{v}}_1$, $\mathbf{\tilde{v}}_2, \ldots, \mathbf{\tilde{v}}_n \in \mathbb{R}^d$ independently from the $d$-dimensional unit sphere. Let $\mathbf{v}_1 = \mathbf{\tilde{v}}_1$. For $i = 2, 3, \ldots$, to generate $\mathbf{v}_i$, we first remove all components of $\mathbf{\tilde{v}}_i$ parallel to $\mathbf{v}_{i-1}$ and then normalize the resultant vector. Formally it can be written as
		 \[
		 	\mathbf{v}_{i} = \frac{\mathbf{\tilde{v}}_i - \langle \mathbf{\tilde{v}}_i, \mathbf{v}_{i-1} \rangle \mathbf{v}_{i-1}}{ \Norm{\mathbf{\tilde{v}}_i - \langle \mathbf{\tilde{v}}_i,\mathbf{v}_{i-1} \rangle \mathbf{v}_{i-1}}}.
		 \]
		We will now prove the statement in the lemma. Consider a pair $i, j$ such that $j > i$ and $j-i \geq 2$. We will show that for a given pair with probability at least $1-F(d)$, the statement in the lemma holds.
		
		Define $\hat{\vec{v}}_i := \mathbf{\tilde{v}}_i - \langle \mathbf{\tilde{v}}_i, \mathbf{v}_{i-1} \rangle \mathbf{v}_{i-1}$. Thus, $\vec{v}_i = \frac{\hat{\vec{v}}_i}{\Norm{\hat{\vec{v}}_i}}$. Consider $\Abs{\langle \vec{v}_i, \vec{v}_j \rangle}$. Using Lemma~\ref{lem:gaussianprojections}, we have that with probability at least $1-\frac{F(d)}{2}$ each of the following holds.
		\begin{enumerate}
			\item $\Abs{\langle \tilde{\vec{v}}_i, \vec{v}_j \rangle} \in \left[ -\frac{C_{\text{conc}}}{3d^{0.25}}, \frac{C_{\text{conc}}}{3d^{0.25}} \right]$
			\item $ \Abs{\langle \tilde{\vec{v}}_i, \vec{v}_{i-1} \rangle} \in  \left[ -\frac{C_{\text{conc}}}{3d^{0.25}}, \frac{C_{\text{conc}}}{3d^{0.25}} \right]$
		\end{enumerate}
		
		Thus, taking a union bound, with probability at least $1-F(d)$ both hold simultaneously. In what follows, we will condition on both these events.
		\begin{align*}
			\Abs{\langle \vec{v}_i, \vec{v}_j \rangle} & = \Abs{\langle \tilde{\vec{v}}_i, \vec{v}_j \rangle} + \Abs{\langle \vec{v}_i - \tilde{\vec{v}_i}, \vec{v}_j \rangle}. & \\
			&\leq \frac{C_{\text{conc}}}{3d^{0.25}} + \Norm{\vec{v}_i-\tilde{\vec{v}}_i}. &\text{(From (1) above and using  Lemma~\ref{appx:innerUnit}, since $\vec{v}_j$ is a unit vector)} \\
			& \leq \frac{C_{\text{conc}}}{3d^{0.25}} + \Norm{\vec{v}_i - \hat{\vec{v}}_i} + \Norm{\hat{\vec{v}}_i - \tilde{\vec{v}}_i}. & \text{(From triangle inequality)} \\
			& \leq \frac{C_{\text{conc}}}{3d^{0.25}} + \Norm{\vec{v}_i - \hat{\vec{v}}_i} + \Abs{\langle \tilde{\vec{v}}_i, \vec{v}_{i-1} \rangle} & \text{(From the definition of $\hat{\vec{v}}_i$)} \\
			& \leq \frac{2C_{\text{conc}}}{3d^{0.25}}  + \Norm{\vec{v}_i - \hat{\vec{v}}_i} & \text{(From (2) above)}
		\end{align*}
		We will now show that $\Norm{\vec{v}_i - \hat{\vec{v}}_i} \leq \frac{C_{\text{conc}}}{3d^{0.25}}$. This will complete the proof. Note from the definition of $\hat{\vec{v}}_i$ we have that $\Norm{\vec{v}_i - \hat{\vec{v}}_i} = \Norm{\frac{\hat{\vec{v}}_i}{\Norm{\hat{\vec{v}}_i}} -\hat{\vec{v}}_i} = \Norm{\hat{\vec{v}}_i}\left( 1- \frac{1}{\Norm{\hat{\vec{v}}_i}} \right)$. Moreover we have that $\Norm{\hat{\vec{v}}_i} \leq \Norm{\tilde{\vec{v}}_i} + \Abs{\langle \tilde{\vec{v}}_i, \vec{v}_{i-1} \rangle}$. From (2) above we have that the second term lies in $\left[ -\frac{C_{\text{conc}}}{3d^{0.25}}, \frac{C_{\text{conc}}}{3d^{0.25}} \right]$. The first term is $1$ since it is a unit vector. Thus $\Norm{\hat{\vec{v}}_i}\left( 1- \frac{1}{\Norm{\hat{\vec{v}}_i}} \right) \leq \frac{C_{\text{conc}}}{3d^{0.25}}$.
	\end{proof}	

\begin{proof}[Proof of Corollary~\ref{cor:Omega}]
		Using Lemma~\ref{lem:orthoSphere} we have that with probability at least $1-F(d)$, the diagonal elements of $\mathbf{\Omega}$ are $1$ and the non-diagonal elements are in the interval $\mathcal{I}_d$. Therefore, we have that with probability at least $1-n^2F(d)$, $\mathbf{\Omega}$ is a matrix with $\Omega_{i, i}=1$ for all $i \in [n]$ and $\Omega_{i, j} \in \mathcal{I}_d$ when $|i-j| \geq 2$. 
\end{proof}

\begin{proof}[Proof of Lemma~\ref{lem:boundLambda}]
		We first compute the expected value $\mathbb{E}[\Lambda_{i, j}]$. From Equation~\eqref{eq:Lambdaij} we have that $\Lambda_{i, j} = \Lambda_{i, i+1}\cdot \Lambda_{i+1, i+2} \cdot \ldots \cdot \Lambda_{j-1, j}$. Note that each of these are independent $\mathcal{U}[-h, h]$ random variables. Therefore, their expected value is $0$ and from independence we have that $\mathbb{E}[\Lambda_{i, j}]=0$.
		
		Let us compute $\Var[\Lambda_{i, j}]$. From the definition of variance and the fact that its mean is $0$, we have that $\Var[\Lambda_{i, j}] = \mathbb{E}[\Lambda^2_{i, i+1}] \cdot \mathbb{E}[\Lambda^2_{i+1, i+2}] \cdot \ldots \cdot \mathbb{E}[\Lambda^2_{j-1, j}] = \sigma_h^{2(j-i)}$. Since we have that $j-i \geq d^{0.1}$ and that $\sigma_h < 1$ therefore, $\Var[\Lambda_{i, j}] \leq \sigma_h^{2d^{0.1}}$. From Chebyshev's inequality~\cite{mitzenmacher2005probability}, we have that for any random variable $X$, $\Pr{ \Abs{X - \mu} \geq a} \leq \frac{\Var{X}}{a^2}$. Lemma~\eqref{lem:boundLambda} follows by the application of Chebyshev's inequality on the random variable $\Lambda_{i, j}$ with $a=\sigma_h^{d^{0.1}/2}$ and using the fact that $d\geq 4^{10} \log^{10} n$.
\end{proof}

\subsection{What happens to the random model when $h > 1$?\\}

We now briefly show that when $h > 1$, the conditions of Model~\ref{mod:localDominance} \emph{do not} hold. More specifically, we can show that the value of $\Diag$ is larger than $1$ with non-negligible probability. Note that the following theorem implies that in this regime we cannot hope to expect $\zeta_h < 1$.

\begin{theorem}
	\label{thm:diverge}	
	For every $h = (1 + z)$, where $z > 0$ is a constant and a large enough $d \geq \Omega(\poly \log n)$, there exists an $i \in [n]$ such that for a constant (dependent on $z$), $0 < C_{z} \leq 1$ we have that,
		\begin{equation}
		\label{eq:negModelDiverge}
		\mathbb{P}\left[\Abs{\Sigma_{i, i+1}} \geq \Sigma_{i,i} \right] \geq C_{z}.
	\end{equation}
\end{theorem}

\begin{proof}
	Using Equation~\eqref{eq:SigmaExpansion} and Corollary~\ref{cor:Omega}, we have the following statements with probability at least $1-\frac{1}{\poly(n)}$.
	\begin{equation}
		\label{eq:ii1L}
		\Abs{\Sigma_{i, i+1}} \geq \sum_{k=1}^{i} \Abs{\Lambda_{k, i} \Lambda_{k, i+1}} - \sum_{k=1}^{i} \sum_{k \neq k':k'=1}^{i+1} C_{\text{conc}} \frac{\Abs{\Lambda_{k, i} \Lambda_{k', i+1}}}{d^{0.25}},
	\end{equation}
	\begin{equation}
		\label{eq:iiL}
		\Sigma_{i, i} \leq \sum_{k=1}^{i} \Lambda^2_{k, i} + \sum_{k=1}^{i} \sum_{k \neq k': k'=1}^{i} C_{\text{conc}} \frac{\Abs{\Lambda_{k, i} \Lambda_{k', i}}}{d^{0.25}}.
	\end{equation}
	The proof follows almost directly from Equations~\eqref{eq:ii1L} and Equations~\eqref{eq:iiL}. Consider $\Abs{\Sigma_{1, 2}} - \Sigma_{1, 1}$. This can be lower-bounded by $\Abs{\Lambda_{1, 2}} - 2 \Abs{\Lambda_{1, 1} } C_{\text{conc}}/d^{0.25} - \Lambda^2_{1, 1}$. Recall that $\Lambda_{1, 1} = 1$.
	
	The probability that $\Abs{\Lambda_{1, 2}} \geq 1 + 2 C_{\text{conc}}/d^{0.25}$ is at least $\frac{z - 2 C_{\text{conc}}/d^{0.25}}{1+z}$ and thus, with the same probability we have that $\Abs{\Lambda_{1, 2}} - 2 C_{\text{conc}}/d^{0.25} - 1$ is larger than $0$. Therefore, $C_z = \frac{z-C_{\text{conc}}/d^{0.25}}{1+z}$ is a constant whenever $z$ is a constant and $d \geq \Omega(\poly \log n)$ is large enough. 
\end{proof}

\section{Experiments}
\label{sec:experiments}

	In this section, we will describe our numerical and real-world experimental results. For the purposes of this section, we define the following quantity \emph{randomized condition number} as follows. Consider a given data correlation matrix $\mathbf{\Sigma}$ and the corresponding parameter matrix $\mathbf{\Lambda}$. Let $\vec{\Sigma}_{\epsilon}$ be a matrix obtained by adding $\mathcal{N}(0, \epsilon^2)$ independent random variable to each entry in $\vec{\Sigma}$ and let $\tilde{\vec{\Lambda}}$ be the corresponding parameter matrix. Then the randomized condition number is $\mathbb{E} \left[ \frac{\Rel{\mathbf{\Lambda}}{\tilde{\mathbf{\Lambda}}}}{\Rel{\mathbf{\Sigma}}{\tilde{\mathbf{\Sigma}}_\epsilon} } \right]$. We use randomized condition number as a proxy for studying the $\ell_{\infty}$-condition number.
	
	\subsection{Good instances}
	We start off by showing that on most random instances, the randomized condition number is small (\emph{i.e.,} instances are stable). In particular, it is stable \emph{because} it satisfies the Model~\ref{mod:localDominance}. Thus, a large fraction of random instances satisfies the condition and hence the stability proof directly implies low condition number on these instances.
	
	The first experiment is as follows. We start with an instance of $(\mathbf{\Lambda}, \mathbf{\Omega})$ and the bow-free path. We construct the matrix $\mathbf{\Sigma}$ using the recurrence \eqref{eqn:Recurrence}. We then plot the values of $|\Sigma_{i, i}|, |\Sigma_{i, i+1}|$, $|\Sigma_{i-1, i}|, |\Sigma_{i-1, i+1}|$. $\mathbf{\Lambda}, \mathbf{\Omega}$ are generated exactly as described in the generative model discussed in the Section~\ref{sec:random} with $h=1$. Figure~\ref{fig:localDominance} shows the plots for three random runs. Note that in all three cases it satisfies the data properties in Model \ref{mod:localDominance}.

		\begin{figure*}[!ht]
			\minipage{0.32\textwidth}
			\includegraphics[width=\linewidth]{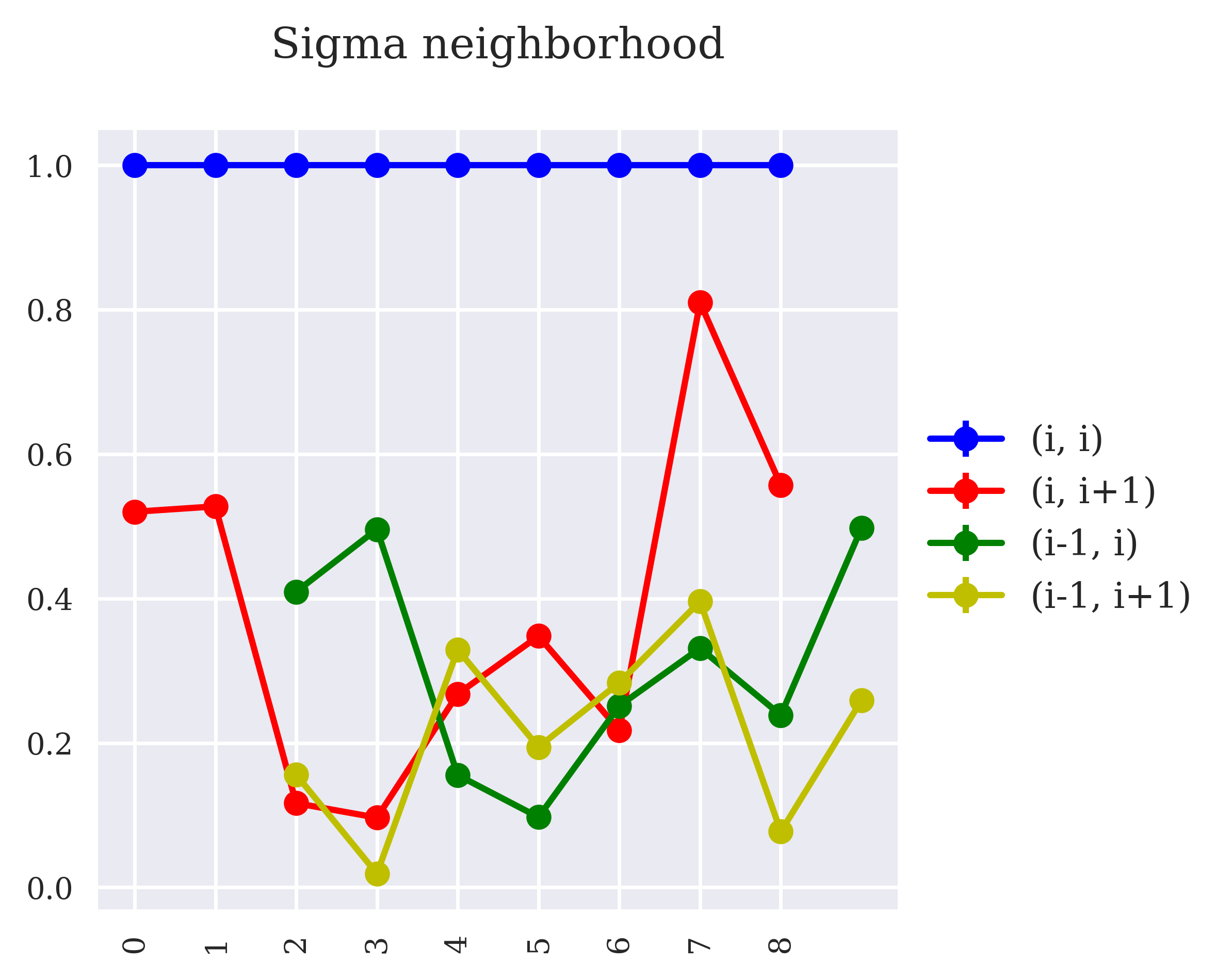}
			\endminipage\hfill
			\minipage{0.32\textwidth}
			\includegraphics[width=\linewidth]{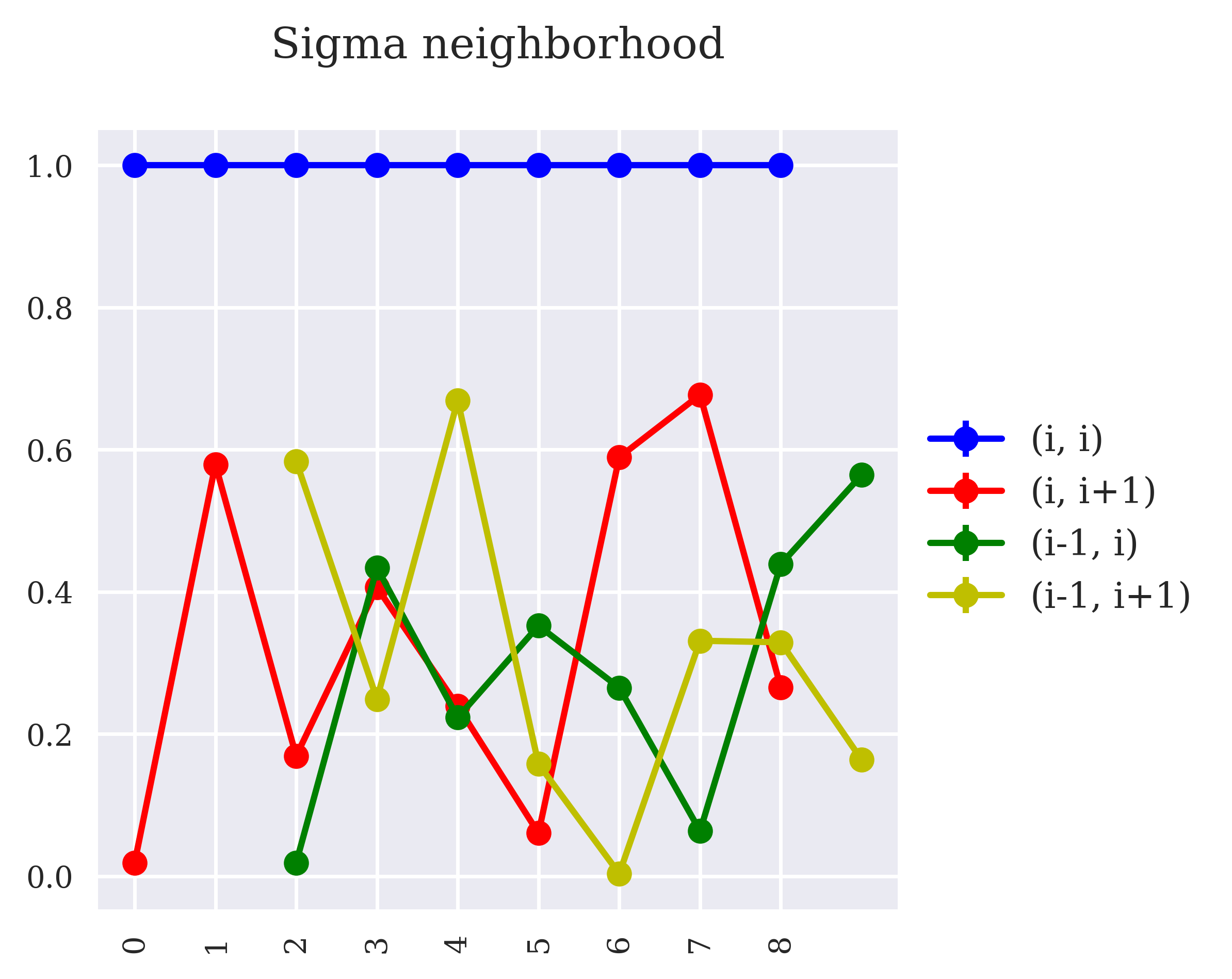}
			\endminipage\hfill
			\minipage{0.32\textwidth}%
			\includegraphics[width=\linewidth]{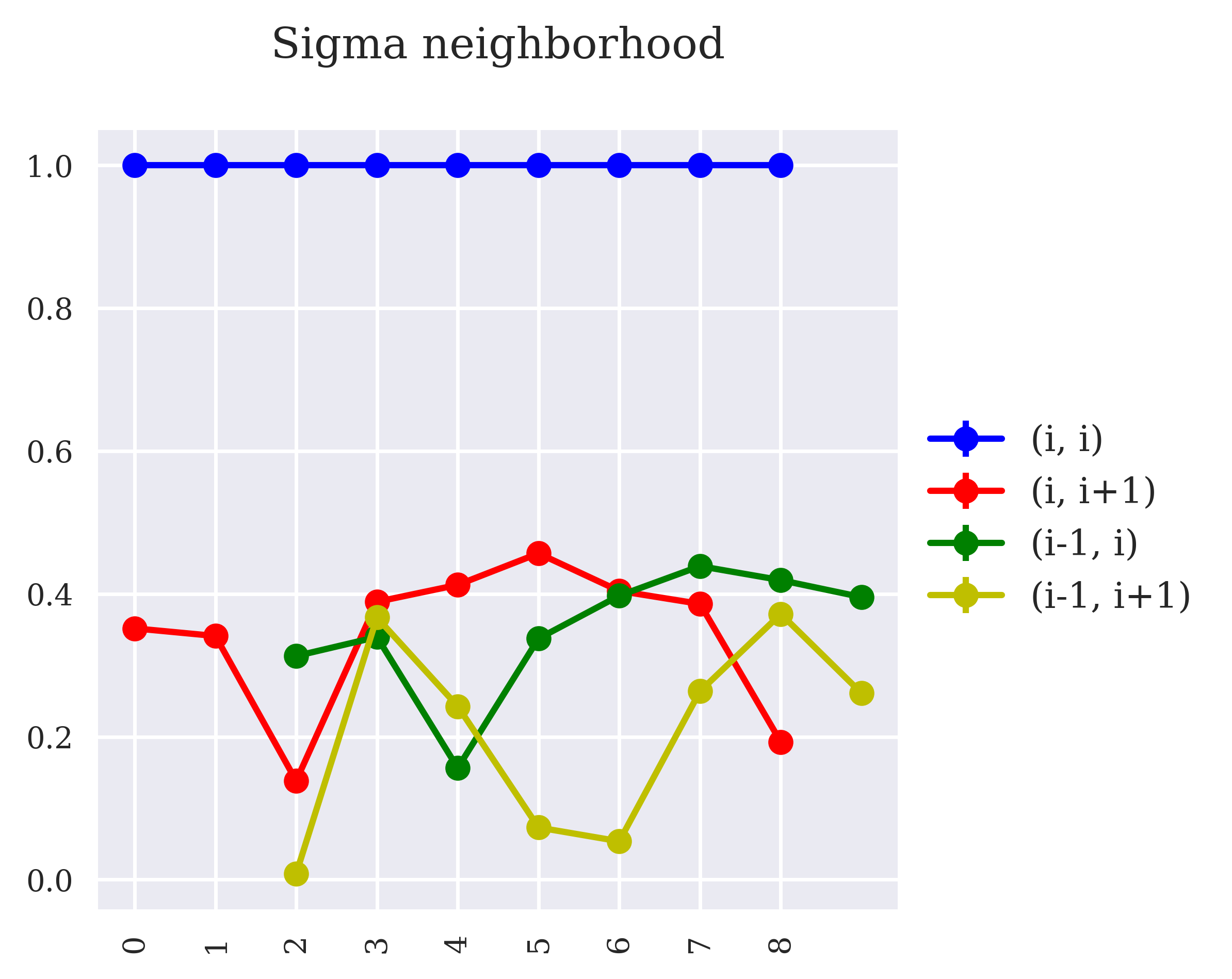}
			\endminipage
			\caption{Local values of $\Sigma$ for three random runs. \textbf{x-axis}: Value of $i$ ($0$-indexed). \textbf{y-axis}: Value of $\vec{\Sigma}$. Thus, a point on the red-line with x-axis label $2$ represents the value $\Sigma_{2, 3}$.}
			\label{fig:localDominance}
		\end{figure*}
	
	Next, we explicitly analyze the effect of \emph{random} perturbations on these instances. As predicted by our theory, the instances are fairly stable. We do this as follows. Given a $(\vec{\Lambda}, \vec{\Omega})$ pair, we can generate $\vec{\Sigma}$ in two ways. We can either (1) use the recurrence~\eqref{eqn:Recurrence}, which can introduce numerical precision errors in $\vec{\Sigma}$, or (2) we can generate samples of $\vec{X}$ (observational data) and estimate $\vec{\Sigma}$ by taking the average over samples of $\vec{X}\vec{X}^T$, which can introduce sampling errors in $\vec{\Sigma}$. 
	
	We add, to each non-zero entry of the obtained $\mathbf{\Sigma}$ (\emph{e.g.,} perturbations), an independent $\mathcal{N}(0, \epsilon^2)$ random variable. The value of $\epsilon$ we chose for this experiment is $10^{-6}$. We then recover $\mathbf{\tilde{\Lambda}}$ from this perturbed $\mathbf{\tilde{\Sigma}}$ using the recurrence \eqref{eqn:Recurrence}. Hence, the ``noise'' entering the system is through two sources, namely, sampling errors and perturbations. Figure~\ref{fig:conditionGood} shows the effect of both sampling and perturbations.  In Figure~\ref{fig:conditionGoodSampling} we show the effect of only sampling errors while in Figure~\ref{fig:conditionGoodPerturbation} we show the effect of only perturbation error.
	
	\begin{figure*}[!ht]
		\minipage{0.49\textwidth}
			\includegraphics[width=\linewidth]{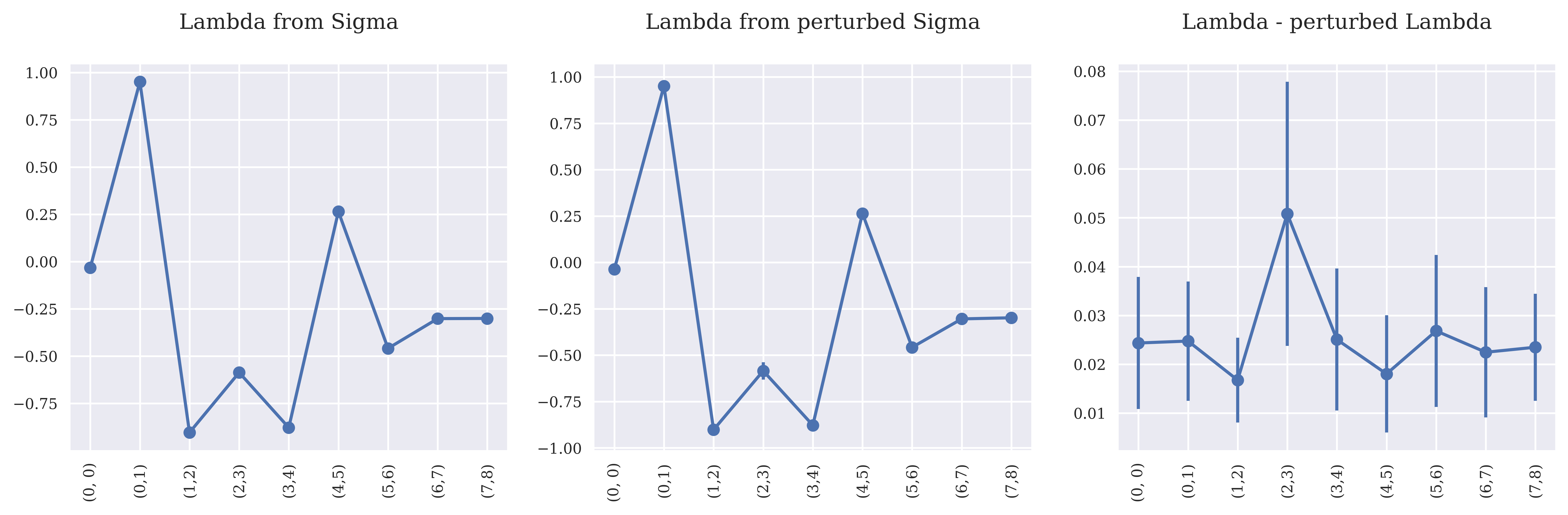}
		\endminipage\hfill
		\minipage{0.49\textwidth}
			\includegraphics[width=\linewidth]{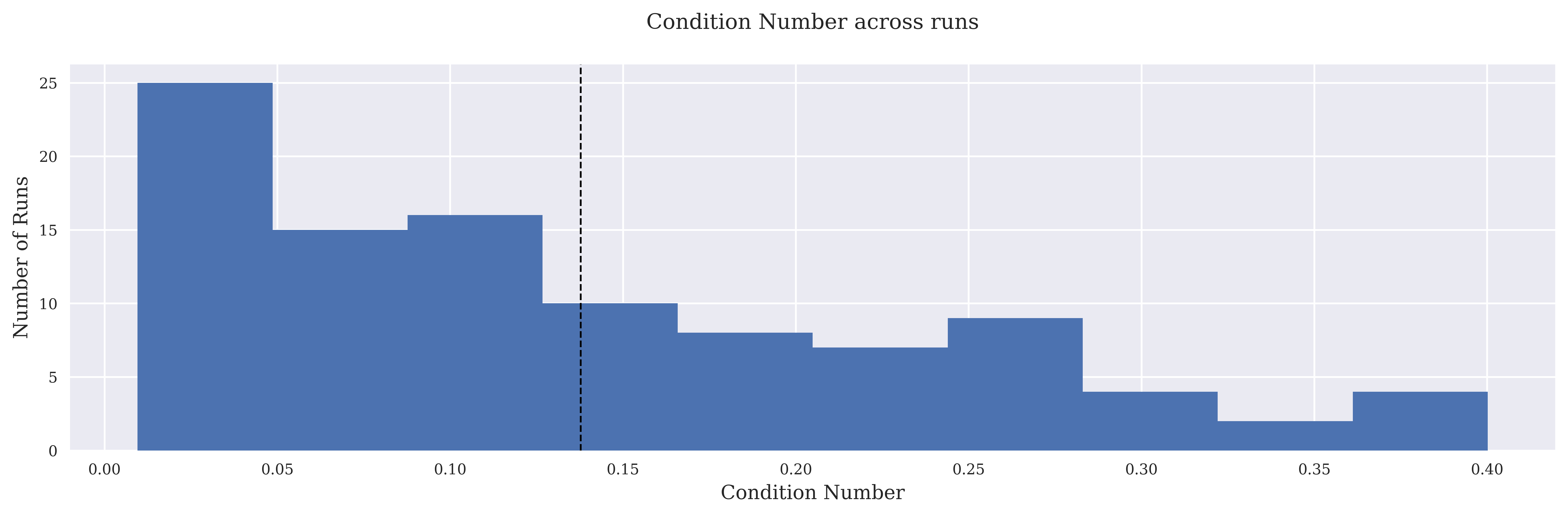}
		\endminipage\hfill
		\caption{Both sampling and perturbation errors. (First) Denotes the actual values of $\mathbf{\Lambda}$. (Second) Denotes the values of $\mathbf{\Lambda}$ obtained from perturbing $\mathbf{\Sigma}$. (Third) Plots the difference between the two value of $\mathbf{\Lambda}$. (Fourth) Gives the histogram of the randomized condition number in various runs.
			(First three plots): \textbf{x-axis}: index of $\Lambda_{i, i+1}$ ($0$-indexed). \textbf{y-axis}: Value of $\vec{\Lambda}$.
		}
		\label{fig:conditionGood}
		\end{figure*}
		
		\begin{figure*}[!ht]
		\minipage{0.49\textwidth}
			\includegraphics[width=\linewidth]{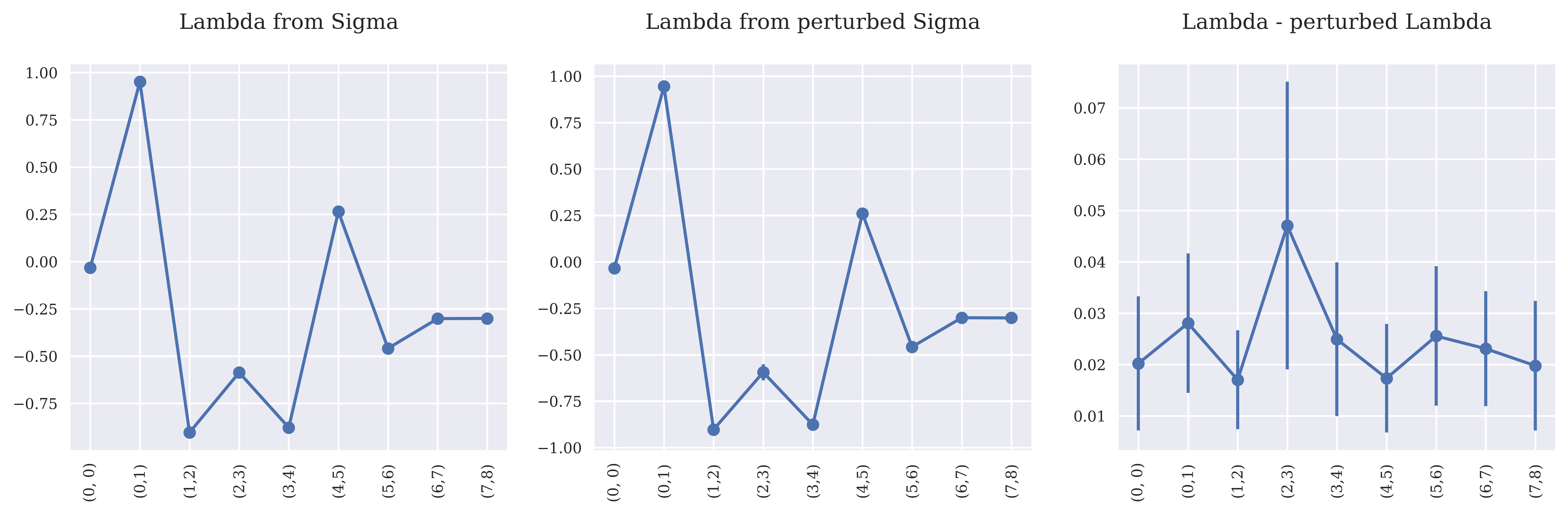}
		\endminipage\hfill
		\minipage{0.49\textwidth}
			\includegraphics[width=\linewidth]{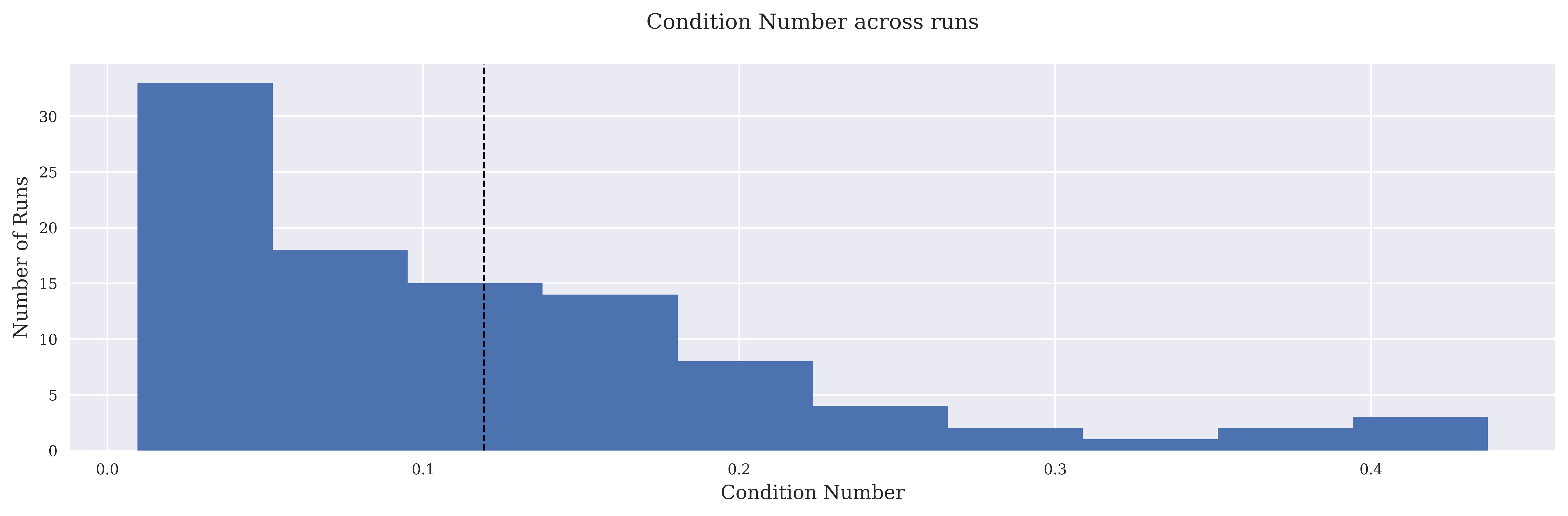}
		\endminipage\hfill
		\caption{Only sampling errors. (First) Denotes the actual values of $\mathbf{\Lambda}$. (Second) Denotes the values of $\mathbf{\Lambda}$ obtained from perturbing $\mathbf{\Sigma}$. (Third) Plots the difference between the two value of $\mathbf{\Lambda}$. (Fourth) Gives the histogram of the randomized condition number in various runs.
			(First three plots): \textbf{x-axis}: index of $\Lambda_{i, i+1}$ ($0$-indexed). \textbf{y-axis}: Value of $\vec{\Lambda}$.
		}
		\label{fig:conditionGoodSampling}
		\end{figure*}
		
		\begin{figure*}[!ht]
		\minipage{0.49\textwidth}
			\includegraphics[width=\linewidth]{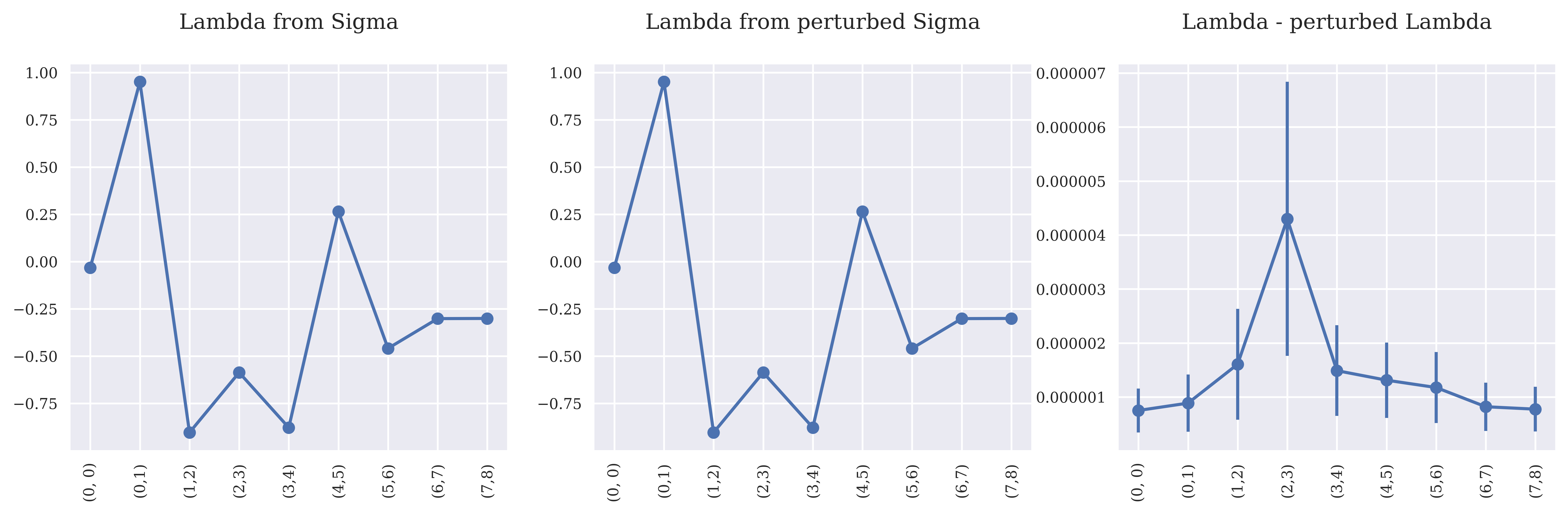}
		\endminipage\hfill
		\minipage{0.49\textwidth}
			\includegraphics[width=\linewidth]{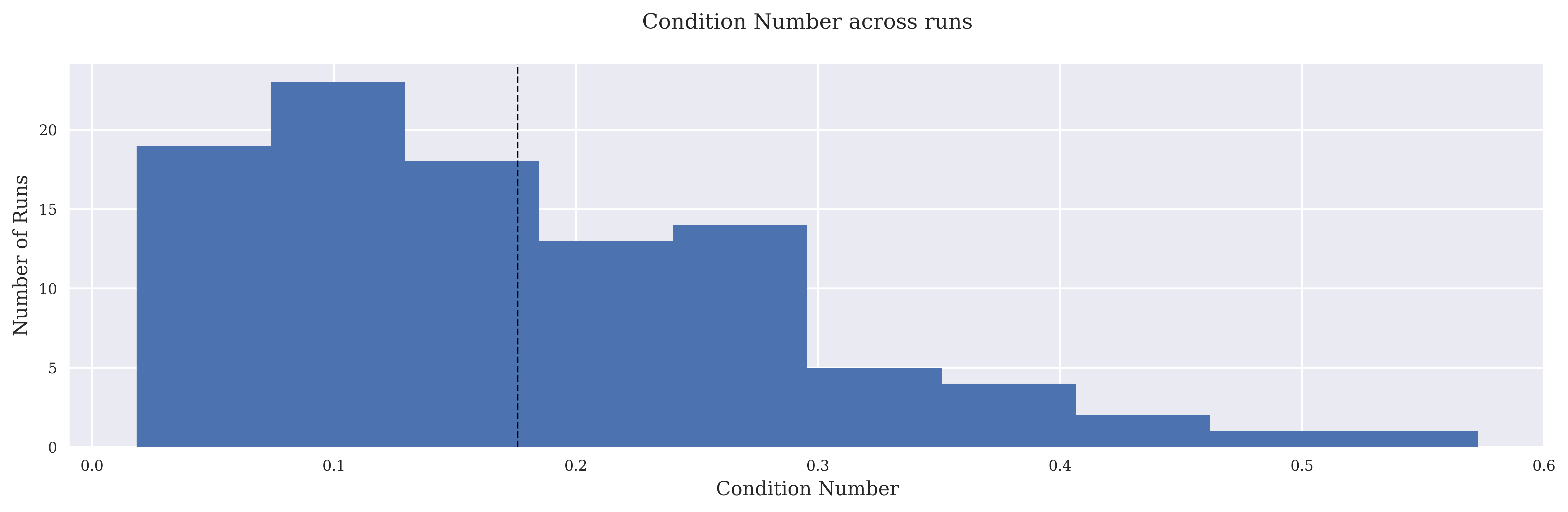}
		\endminipage\hfill
		\caption{Only perturbation errors. (First) Denotes the actual values of $\mathbf{\Lambda}$. (Second) Denotes the values of $\mathbf{\Lambda}$ obtained from perturbing $\mathbf{\Sigma}$. (Third) Plots the difference between the two value of $\mathbf{\Lambda}$. (Fourth) Gives the histogram of the randomized condition number in various runs.
			(First three plots): \textbf{x-axis}: index of $\Lambda_{i, i+1}$ ($0$-indexed). \textbf{y-axis}: Value of $\vec{\Lambda}$.
		}
		\label{fig:conditionGoodPerturbation}
		\end{figure*}

	\subsection{Bad Instances}
	In the next experiment, we start off with the bad example on a bow-free path of length 4. We show that as confirmed in the theory, this instance is highly unstable. We then perturb this instance slightly and show that in a small enough ball around this instance, it continues to remain unstable. Finally, we plot a graph showing the variation of the randomized condition number as a function of the radius of the ball around this instance (see next paragraph for precise definition).
	
	In particular, we start with the bad $\mathbf{\Lambda}$ and $\mathbf{\Omega}$. We then consider a region around these matrices as follows. To $\mathbf{\Lambda}$ (and likewise to $\mathbf{\Omega}$) add an independent $\mathcal{N}(0, \epsilon^2)$ random variable to each non-zero entry. We then consider the effect of random perturbations starting from this new $(\mathbf{\tilde{\Lambda}}, \mathbf{\Omega})$ pair. Figure~\ref{fig:badLambda} shows the effect of random perturbations in the region around the matrix $\mathbf{\Lambda}$ and Figure~\ref{fig:badOmega} shows the same in the region around $\mathbf{\Omega}$.
		As evident from the figures, the randomized condition number continues to remain large even when slightly perturbed. Hence, this implies that there are infinitely many pairs $(\mathbf{\Lambda}, \mathbf{\Omega})$ which produce very large condition numbers.
		
	\begin{figure*}[!ht]
		\minipage{0.45\textwidth}
			\includegraphics[width=\linewidth]{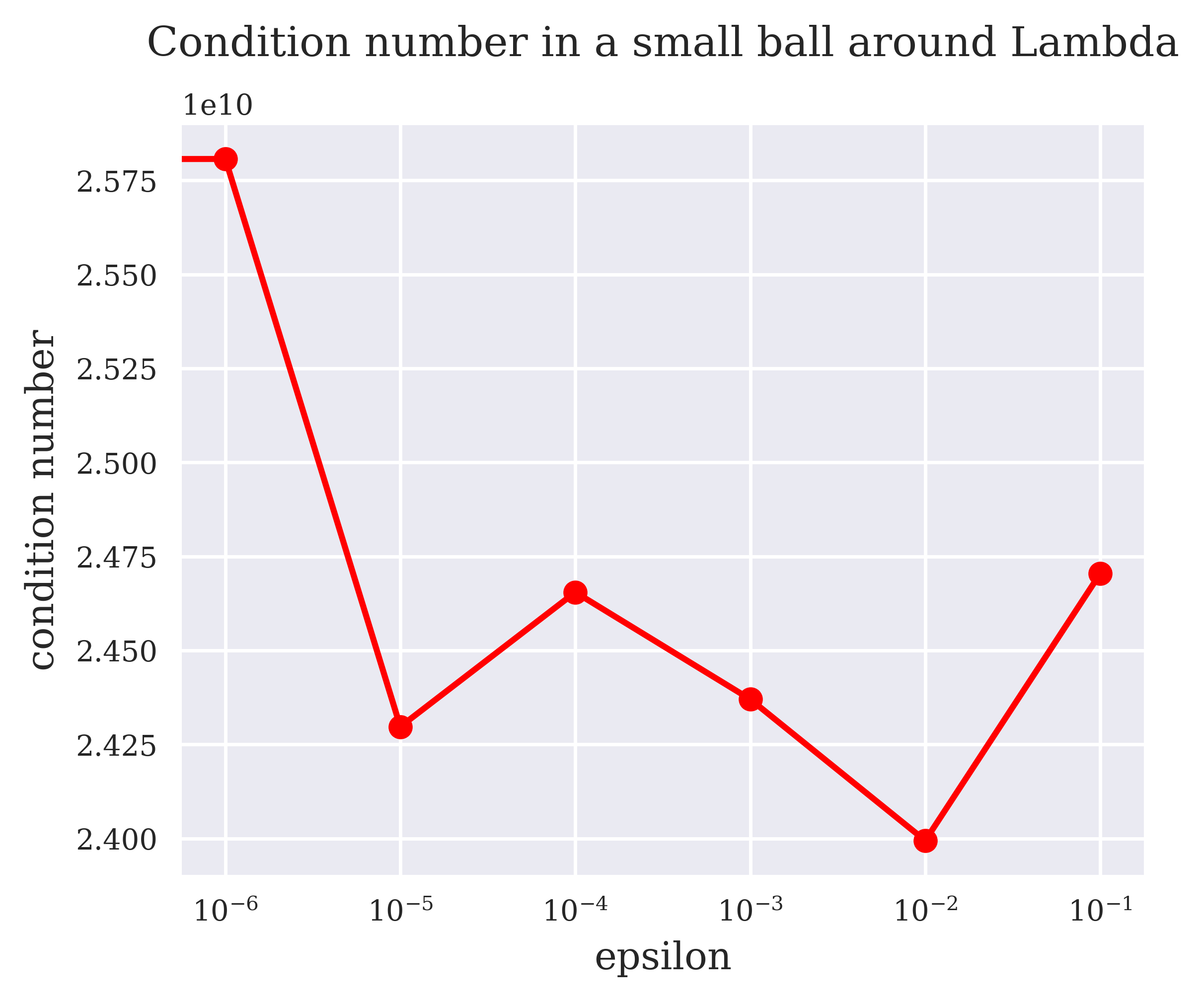}
			\caption{Randomized condition number in a small region around the bad $\mathbf{\Lambda}$. Scale of \textbf{$y$-axis}: $10^{10}$.}
		\label{fig:badLambda}
		\endminipage\hfill
		\minipage{0.45\textwidth}
		\includegraphics[width=\linewidth]{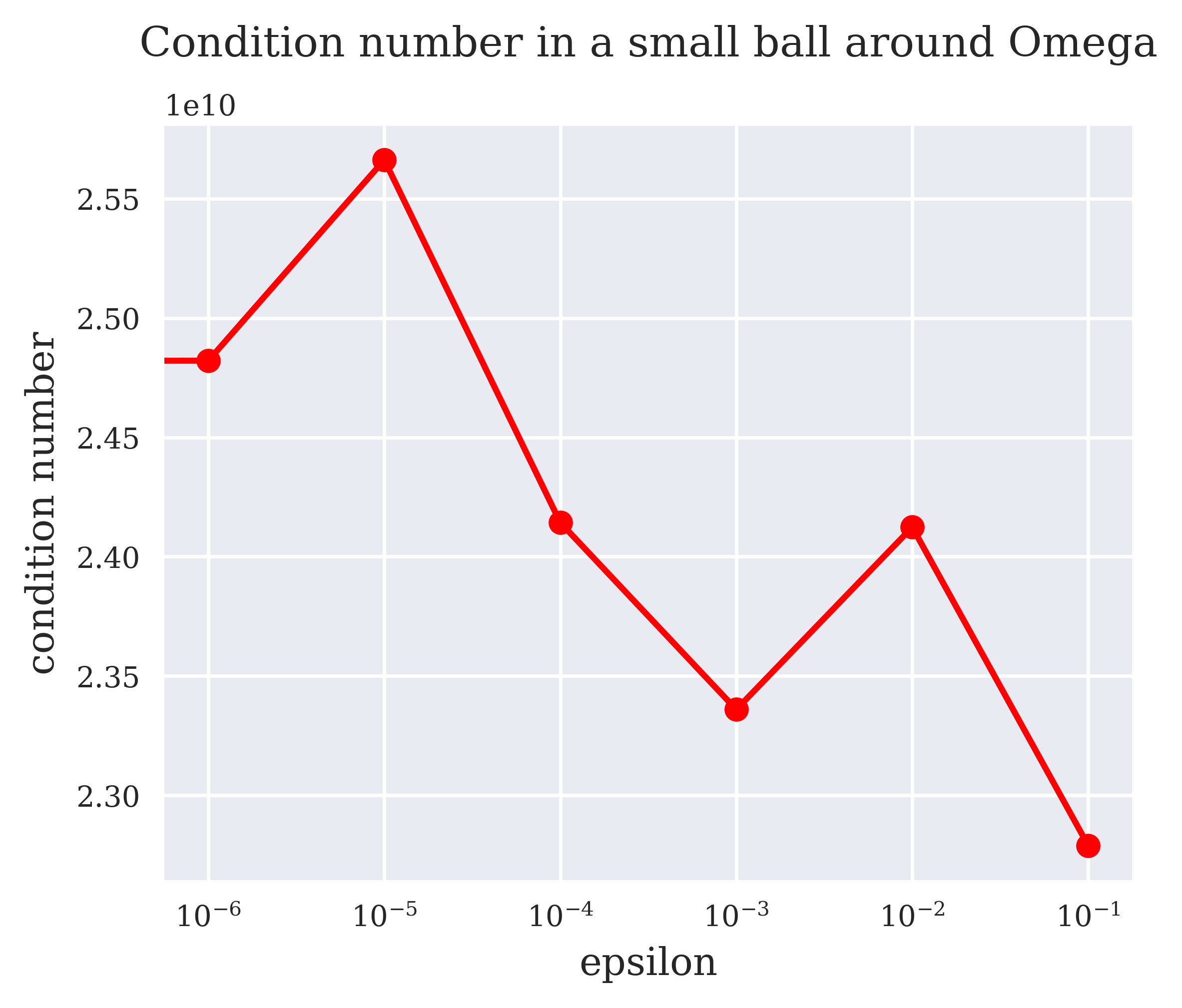}
		\caption{Randomized condition number in a small region around the bad $\mathbf{\Omega}$. Scale of \textbf{$y$-axis}: $10^{10}$.}
		\label{fig:badOmega}
		\endminipage\hfill
	\end{figure*}
	\begin{figure*}
		\includegraphics[width=0.5\linewidth]{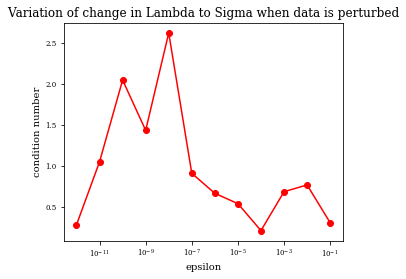}
			\caption{Variation of the randomized condition number as a function of the perturbation error on the sociology dataset}
			\label{fig:realWorld}
		\end{figure*}

	\subsection{Real-world data}
		\label{subsec:real}
		In this experiment, we analyze a sociology dataset on $6$ vertices, taken from a public Internet repository~\cite{dataset}. This was used in \cite{shimizu2011directlingam} which almost resembles our model of $\LSEM$ with the exception that we consider Gaussian noise while they don't. Nonetheless, we experiment with the matrices returned by their algorithm, and compute the randomized condition number. 
		
		We pre-process the dataset such that the setup is almost similar to that in their work\footnote{In a private communication, the authors of \cite{shimizu2011directlingam} provided exact data used in \cite{shimizu2011directlingam}}. Note that the causal graph given by domain experts is \emph{bow-free} therefore, from the theorem in \cite{brito2002new}, this graph is identifiable. We constructed the $\mathbf{\Sigma}$ matrix from the observational data. We then use the algorithm of \cite{FDD2012Annals} to recover the parameter $\mathbf{\Lambda}$. We then \emph{perturb} the data as follows. To each entry in the observational data, we add an additive $\mathcal{N}(0, \epsilon^2)$ noise independently. Compared to the magnitudes of the actual data, this additive noise is insignificant. We recompute $\mathbf{\tilde{\Sigma}}$, from this perturbed observational data. Using the algorithm of \cite{FDD2012Annals}, we once again recover the parameter $\mathbf{\tilde{\Lambda}}$. For a given $\epsilon$, we run $100$ independent runs and take the average. Figure~\ref{fig:realWorld} shows the variation of the randomized condition number as a function of $\epsilon$. As the value of $\epsilon$ becomes very small, the randomized condition number remains almost constant and approaches $10^{-1}$. This number is very close to our \emph{well-behaved} instances implying that this dataset is fairly robust when modeled as a $\LSEM$.

\subsection{Other DAGs}
		In this section, we run further experiments which considers other bow-free graphs where the DAG is not a path. The first class of graphs we consider is the \emph{clique of paths}. In this class, the unobservable edges satisfy the \emph{bow-free} 
		property\footnote{recall from the definition that this implies that for a pair of vertices, both observable and unobservable edges do not exist simultaneously}. The observable edges can be described as follows. We have a parameter $k$ which controls the size of the \emph{clique}. The DAG is layered with $\tfrac{n}{k}$ layers and every layer has $k$ vertices. Directed edges are placed from every vertex in layer $i$ to every other vertex in layer $i+1$ for $i \in [n/k-1]$. Thus, the induced subgraph (ignoring the direction of the edges) on the vertices in layer $i$ and layer $i+1$ form a clique. The bow-free paths can be viewed as a member in this family with $k=1$. As before, the non-zero entries of $\vec{\Lambda}$ is chosen as independent and uniform samples from $\mathcal{U}[-1, 1]$. To generate $\vec{\Omega}$ we first generate vectors $\vec{v}_1, \vec{v}_2, \ldots, \vec{v}_n$, such that $v_{i}$ is a uniform sample from the subspace perpendicular to the one spanned by $\vec{v}_{i-1}, \vec{v}_{i-2}, \ldots, \vec{v}_{i-k}$. Then we set $\Omega_{i, j} = \langle \vec{v}_i, \vec{v}_j \rangle$. We run experiments with $n=20, k=2$ and $n=30, k=5$ with results in Figures~\ref{fig:n20k2} and \ref{fig:n30k5} respectively. As in the case of paths, the instances have a low randomized condition number.

		\begin{figure*}[!ht]
		\minipage{0.49\textwidth}
			\includegraphics[width=\linewidth]{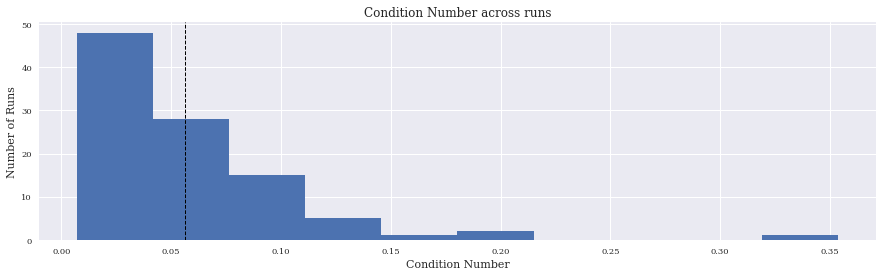}
			\caption{Randomized condition number for clique of paths with $n=20$ and $k=2$.}
		\label{fig:n20k2}
		\endminipage\hfill
		\minipage{0.49\textwidth}
		\includegraphics[width=\linewidth]{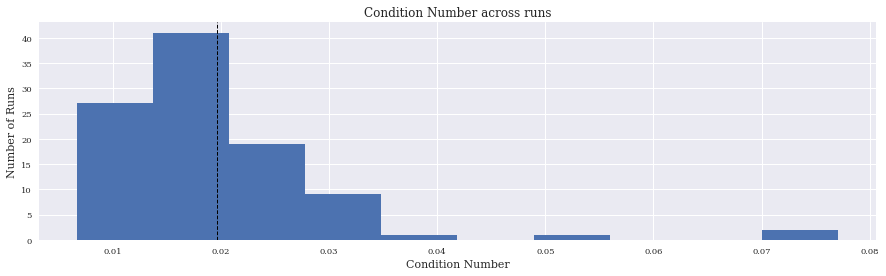}
		\caption{Randomized condition number for clique of paths with $n=30$ and $k=5$.}
		\label{fig:n30k5}
		\endminipage\hfill
		\end{figure*}

		The next class of graphs we consider is the \emph{layered} graph. We generate these instances as follows. We start with the graph of $n=30$ and $k=5$ from the previous experiment on clique of paths. For a parameter $p$, every edge is independently dropped with probability $p$. Thus, not all edges in any clique is present in the final causal model. We run experiments for $p=0.2, 0.5, 0.8$ with results in Figures~\ref{fig:n30k5p0.2}, \ref{fig:n30k5p0.5} and \ref{fig:n30k5p0.8} respectively. 
		
		\begin{figure*}[!ht]
		\minipage{0.33\textwidth}
			\includegraphics[width=\linewidth]{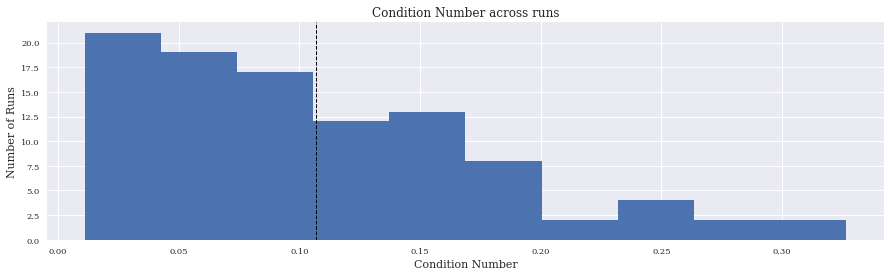}
			\caption{Layered graph: $p=0.2$.}
		\label{fig:n30k5p0.2}
		\endminipage\hfill
		\minipage{0.33\textwidth}
		\includegraphics[width=\linewidth]{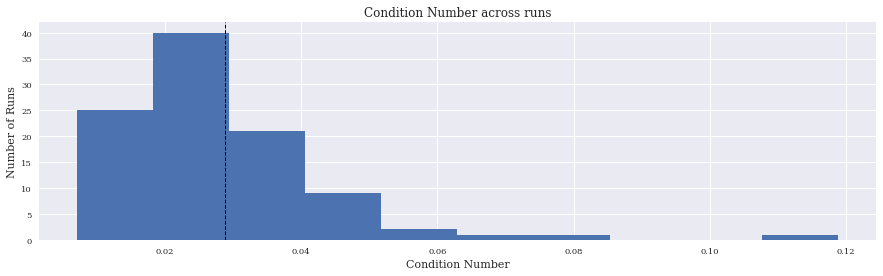}
		\caption{Layered graph: $p=0.5$.}
		\label{fig:n30k5p0.5}
		\endminipage\hfill
		\minipage{0.33\textwidth}
		\includegraphics[width=\linewidth]{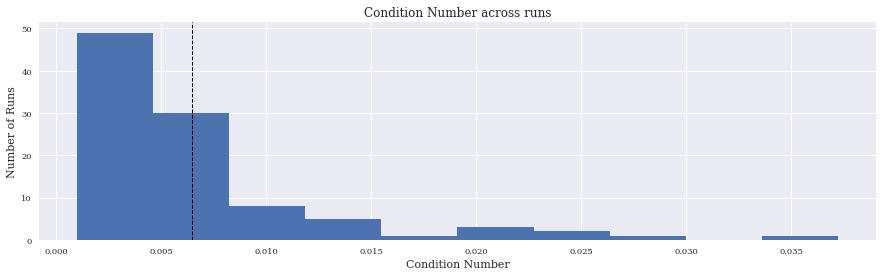}
		\caption{Layered graph: $p=0.8$.}
		\label{fig:n30k5p0.8}
		\endminipage\hfill
		\end{figure*}

\section{General Procedure for Practitioners}
	\label{sec:generalAlgorithm}
	
	In this section, we conclude this work by giving a general heuristic which practitioners using the $\LSEM$ can incorporate to identify \emph{good} and \emph{bad} instances (in terms of the randomized condition number). As suggested by theory and experiments, a vast range of instances are well-conditioned while some \emph{carefully} constructed instances can be arbitrarily ill-conditioned. Algorithm~\ref{alg:generalAlgorithm} describes a general procedure for identifying if a given instance at hand has a high randomized condition number (as defined in section~\ref{sec:experiments}). We use this procedure in our experiments on the real-world dataset (sub-section~\ref{subsec:real}) to conclude that the given instance is well-conditioned.

	\begin{algorithm2e}[!h]
		\caption{General procedure to check if a given instance is well-conditioned}
		\label{alg:generalAlgorithm}
		\DontPrintSemicolon
		\SetKwInOut{Input}{input}
		\Input{Data correlation matrix $\vec{\Sigma}$ of size $n \times n$, ill-condition threshold $\tau$.}
		
		Set $\epsilon = \frac{1}{n^4}$ \;
		\For{$r = 1, 2, \ldots, n^4$}{
			\begin{minipage}{15cm}
			\begin{enumerate}
			\item
			\textbf{Random perturbation.} Construct $\vec{\tilde{\Sigma}}$ by adding a $\mathcal{N}(0, \epsilon^2)$ random variable independently to each entry in $\vec{\Sigma}$. \;
	
			\item 
			\textbf{Recovery.} Recover parameters $\vec{\tilde{\Lambda}}$, $\vec{\Lambda}$ from $\vec{\tilde{\Sigma}}$ and $\vec{\Sigma}$ respectively. \;
			
			\item 
			\textbf{Condition Number.} Compute the condition number for this run $\kappa_r := \frac{\Abs{\vec{\tilde{\Lambda}} - \vec{\Lambda}}_\infty}{\Abs{\vec{\tilde{\Sigma}} - \vec{\Sigma}}_\infty}$. \;
			\end{enumerate}
			\end{minipage}
		 }
		 Compute the average $\kappa$ of the condition numbers $\kappa_1, \kappa_2, \ldots, \kappa_{n^4}$. If $\kappa \geq \tau$ then the instance is \emph{ill-conditioned}. Else it is well-conditioned. \;
		 
	\end{algorithm2e}

\section*{Acknowledgements}
The authors would like to thank Amit Sharma, Ilya Shpitser and Piyush Srivastava for useful discussions on causality. The authors would also like to thank Shohei Shimizu for providing us with the sociology dataset. 

Part of this work was done when Karthik Abinav Sankararaman was visiting Indian Institute of Science, Bangalore. KAS was supported in part by NSF Awards CNS 1010789, CCF 1422569, CCF-1749864 and research awards from Adobe, Amazon, and Google. Anand Louis was supported in part by SERB Award ECR/2017/003296. AL is also grateful to Microsoft Research for supporting this collaboration.

\bibliographystyle{plain}
\bibliography{../references}

\appendix
\section{Proofs from Prior Work}
\label{sec:priorWorkProofs}

\subsection{Proof of Recurrence~\eqref{eqn:Recurrence}}
\begin{proof}[\emph{Proof(\cite{FDD2012Annals})}]
We verify that this recurrence is correct from first principles. Let $X_1, X_2, \ldots, X_n$ denote the random variables for the $n$ vertices. Recall that $\Sigma_{i, j} = \mathbb{E}[X_i X_j]$. Let $N_1, N_2, \ldots, N_n$ denote the noise variables. From the linear SEM model we have that $X_1 = N_1$. And for every $i \geq 2$ we have that $X_i = \Lambda_{i-1, i} X_{i-1} + N_i$.

Consider $\Sigma_{1, 2} = \mathbb{E}[X_1 X_2]$. This can be written as $\mathbb{E}[X_1 (\Lambda_{1, 2} X_1 + N_2)] = \Lambda_{1, 2} \Sigma_{1, 1} + \mathbb{E}[X_1 N_2] = \Lambda_{1, 2} \Sigma_{1, 1} + \Omega_{1, 2} = \Lambda_{1, 2} \Sigma_{1,1}$. The last inequality is from the zero-patterns in $\Omega$. Consider $i \geq 2$. We will show that the recurrence is correct by showing that the following two quantities are equal.
	\begin{equation*}
			-\Lambda_{i, i+1}\Lambda_{i-1, i} \Sigma_{i-1, i} + \Lambda_{i, i+1} \Sigma_{i, i} = -\Lambda_{i-1, i} \Sigma_{i-1, i+1} + \Sigma_{i, i+1}  
	\end{equation*}
	Consider the first term in the RHS. Note that $\Sigma_{i-1, i+1} = \mathbb{E}[X_{i-1} X_{i+1}]$. From the linear SEM this can be expanded as $\mathbb{E}[X_{i-1} (\Lambda_{i, i+1} X_i + N_{i+1})] = \Lambda_{i, i+1} \Sigma_{i-1, i} + \mathbb{E}[X_{i-1} N_{i+1}]$.
	
	The second term in the RHS can be expanded as follows. $\Sigma_{i, i+1} = \mathbb{E}[X_i (\Lambda_{i, i+1} X_i + N_{i+1})] = \Lambda_{i, i+1} \Sigma_{i, i} + \mathbb{E}[X_i N_{i+1}]$. We further have $\mathbb{E}[X_i N_{i+1}] = \mathbb{E}[(\Lambda_{i-1, i} X_{i-1} + N_i) N_{i+1}] = \Lambda_{i-1, i} \mathbb{E}[X_{i-1} N_{i+1}] + \Omega_{i, i+1}$. 
	
	Therefore the RHS evaluates to $-\Lambda_{i-1, i}\Lambda_{i, i+1} \Sigma_{i-1, i} + \Lambda_{i, i+1} \Sigma_{i, i}$ which is equal to the LHS.
\end{proof}

\section{Technical Lemmas}
	\label{sec:technicalLemmas}
	
We use the following facts about the uniform distribution and vectors sampled uniformly from a unit sphere.

\begin{lemma}[Uniform distribution]
	\label{appx:uniform}
	Let $X$ be a random variable that is uniformly distributed in the interval $[-h, h]$. Then we have the following.
	\begin{enumerate}
		\item The mean $\mathbb{E}[X] = 0$ and the variance $\Var=\mathbb{E}[X^2] = h^2/3$.
		\item For any given $0 \leq \beta \leq 1$ we have that $\Pr{-\beta \leq X \leq \beta} = \beta/h$.
	\end{enumerate}
 \end{lemma}	
	The proof of these theorems follow directly from the definition of an uniform distribution and we refer the reader to a standard textbook on probability~\cite{mitzenmacher2005probability}.
	
	We state a Lemma on the gaussian behavior of unit vectors. This can be found in Lemma 5 of \cite{arora2009expander}.
	\begin{lemma}[Gaussian behavior of projections]
		\label{lem:gaussianprojections}
		Let $\mathbf{v}$ be an arbitrary unit vector in $\mathbb{R}^d$. Let $\mathbf{u}$ be a randomly chosen unit vector of dimension $d$. Then for every $x \leq \sqrt{d}/4$ we have,
		\[
				\Pr{\Abs{\langle \mathbf{v}, \mathbf{u} \rangle} \geq \frac{x}{\sqrt{d}}	} \leq \exp(-x^2/4).
		\]
		As a corollary, for an absolute constant $C_{\text{conc}} > 0$ and $x = C_{\text{conc}} \cdot d^{0.25}$ we have that,
		\[
				\Pr{\Abs{\langle \mathbf{v}, \mathbf{u} \rangle} \geq \frac{C_{\text{conc}}}{d^{0.25}} } \leq \exp\left[ -\Omega \left( \sqrt{d} \right) \right].
		\]
	\end{lemma}
	
	We also use the following well-known fact about inner products. For completeness, we prove it here.
	
	\begin{lemma}[Inner-product with unit vectors]
		\label{appx:innerUnit}
		Let $\vec{u} \in \mathbb{R}^d$ be an unit vector and let $\vec{v} \in \mathbb{R}^d$ be any arbitrary vector. Then we have the following.
		\[
				\Abs{\langle \vec{v}, \vec{u} \rangle} \leq \norm{\vec{v}}.
		\] 
	\end{lemma}
	\begin{proof}
		Note that from Cauchy-Swarz inequality we have $\Abs{\langle \vec{v}, \vec{u} \rangle} \leq \Norm{\vec{v}} \Norm{\vec{u}}$. Since $\Norm{\vec{u}} =1$ we get the above inequality.
	\end{proof}

\end{document}